\documentclass[hidelinks,onefignum,onetabnum]{siamart250106}

\usepackage{lipsum}
\usepackage{amsfonts}
\usepackage{graphicx}
\usepackage{epstopdf}
\usepackage{algorithmic}
\ifpdf
  \DeclareGraphicsExtensions{.eps,.pdf,.png,.jpg}
\else
  \DeclareGraphicsExtensions{.eps}
\fi

\newsiamremark{remark}{Remark}
\newsiamremark{hypothesis}{Hypothesis}
\crefname{hypothesis}{Hypothesis}{Hypotheses}
\newsiamthm{claim}{Claim}

\headers{Neural Implicit Solution Formula for HJ PDEs}{Y. Park and S. Osher}

\title{Neural Implicit Solution Formula for Efficiently \\Solving Hamilton-Jacobi Equations\thanks{Submitted to the editors DATE %January 29, 2025.
\funding{S. Osher was partially supported by
MURI DOD-AFOSR-N00014-20-1-2787,
NSF NSF-2208272, and NSF STROBE NSF-1554564. Y. Park was supported by the NRF Grant RS-2024-00343226.}}}

\author{Yesom Park\thanks{Department of Mathematics, University of California, Los Angeles 
  (\email{yeisom@math.ucla.edu}).}
\and Stanley Osher\thanks{Department of Mathematics, University of California, Los Angeles
  (\email{sjo@math.ucla.edu}).}}

\usepackage{amsopn}

\ifpdf
\hypersetup{
  pdftitle={Neural Implicit Solution Formula for HJ PDEs},
  pdfauthor={Y. Park and S. Osher}
}
\fi

\usepackage{subfigure}
\usepackage{booktabs} % for professional tables
\usepackage{empheq}
\newtheorem{example}{Example}[section]

\newcommand{\R}{\mathbb{R}}

\newcommand{\cL}{\mathcal{L}}

\newcommand{\cS}{\mathcal{S}}

\newcommand{\cU}{\mathcal{U}}

\newcommand{\bN}{\mathbb{N}}

\newcommand{\bR}{\mathbb{R}}

\newcommand{\dV}[1][u]{\cU\left(\Omega;\bR^{d_v}\right)}
\newcommand{\dVt}[1][u]{\cU\left(\left[0,\infty\right)\times\Omega;\bR^{d_v}\right)}

\newcommand{\bx}{\mathbf{x}}
\newcommand{\by}{\mathbf{y}}
\newcommand{\bz}{\mathbf{z}}
\newcommand{\bv}{\mathbf{v}}
\newcommand{\ba}{\mathbf{a}}
\newcommand{\bb}{\mathbf{b}}

\newcommand{\bp}{\mathbf{p}}
\newcommand{\bq}{\mathbf{q}}

\DeclareMathOperator*{\argmin}{\arg\min}

\newcommand*\diff{\mathop{}\!\mathrm{d}}

\begin{document}

\maketitle
\begin{abstract}
This paper presents an implicit solution formula for the Hamilton-Jacobi partial differential equation (HJ PDE). The formula is derived using the method of characteristics and is shown to coincide with the Hopf and Lax formulas in the case where either the Hamiltonian or the initial function is convex. It provides a simple and efficient numerical approach for computing the viscosity solution of HJ PDEs, bypassing the need for the Legendre transform of the Hamiltonian or the initial condition, and the explicit computation of individual characteristic trajectories. A deep learning-based methodology is proposed to learn this implicit solution formula, leveraging the mesh-free nature of deep learning to ensure scalability for high-dimensional problems. Building upon this framework, an algorithm is developed that approximates the characteristic curves piecewise linearly for state-dependent Hamiltonians. Extensive experimental results demonstrate that the proposed method delivers highly accurate solutions, even for nonconvex Hamiltonians, and exhibits remarkable scalability, achieving computational efficiency for problems up to 40 dimensions.
\end{abstract}



\section{Introduction}
Hamilton-Jacobi partial differential equations (HJ PDEs) are of paramount importance in various fields of mathematics, physics, and engineering, including optimal control \cite{evans1984differential, sideris2005efficient, bansal2017hamilton}, mechanics \cite{denman1973solution, de2008linear}, and the study of dynamic systems \cite{khanin2010particle, rajeev2008hamilton}. As they provide a powerful framework for modeling systems governed by physical laws, HJ PDEs have a wide range of applications in diverse areas such as geometric optics \cite{osher1993level, minano2006hamilton}, computer vision \cite{caselles1997geodesic, gilboa2009nonlocal, osher2007geometric}, robotics \cite{lin1998optimal,lewis2003robot, bansal2020hamilton}, trajectory optimization \cite{delahaye2014mathematical, parzani2018hamilton}, traffic flow modeling \cite{imbert2013hamilton, laval2013hamilton}, and financial strategies \cite{forsyth2007numerical, cao2009optimal}. These applications illustrate the versatility and significance of HJ PDEs, emphasizing the necessity for effective methods to solve them in both theoretical and practical contexts. It is well-known that the solutions to HJ PDEs are typically continuous but exhibit discontinuous derivatives, irrespective of the smoothness of the initial conditions or the Hamiltonian. Moreover, such solutions are typically non-unique. In this regard, viscosity solutions \cite{crandall1983viscosity} are commonly considered as the appropriate notion of solution, as they reflect the physical characteristics inherent to the problem.

Numerical methods for solving HJ PDEs have been extensively developed, with numerous practical applications across various fields. The most prominent methods include essentially non-oscillatory (ENO) and weighted ENO (WENO) type schemes \cite{osher1991high, jiang2000weighted, bryson2003high, qiu2005hermite}, semi-Lagrangian methods \cite{falcone2002semi, cristiani2007fast, falcone2013semi}, and level set approaches \cite{osher1988fronts, osher1993level, osher2004level, mitchell2004computing, ansari2013application}. However, they encounter significant scalability challenges as the dimensionality of the state space increases. These methods rely on discretization of the state space with a grid and approximating the Hamiltonian in a discrete form. Consequently, the number of grid points required to obtain accurate solutions grows exponentially with the dimensionality of the problem, resulting in prohibitive computational costs. In high-dimensional settings, particularly those involving more than four dimensions, this scaling issue renders the classical methods impractical for many real-world applications, where high-dimensional state spaces are prevalent.

Several approaches have been proposed to address the curse of dimensionality in solving HJ PDEs. Methods based on max-plus algebra \cite{mceneaney2006max, akian2008max, fleming2000max} show promise but are restricted to specific classes of optimal control problems and encounter significant challenges in practical implementation due to their complexity. Another promising approach involves the use of Hopf or Lax formulas to represent solutions to HJ PDEs \cite{darbon2016algorithms, chow2017algorithm, chow2019algorithm, chen2024hopf}. These formulas offer a causality-free approach, where solutions at each spatial and temporal point can be computed by solving an optimization problem, thus enabling parallel computation.
This approach eliminates the reliance on grid-based discretization, making it particularly well-suited for high-dimensional problems. However, these methods require computing the Legendre transform of the Hamiltonian or initial function and are generally applicable only under specific assumptions, such as convexity, or when the problem can be framed as a particular type of control problem. 
In parallel, algorithms based on Pontryagin's Maximum Principle \cite{kang2015causality, kang2017mitigating, yegorov2021perspectives}, which employ the method of characteristics, have been proposed. Despite their potential, the practical effectiveness of these methods is often limited by the need to solve a system of ordinary differential equations (ODEs) at each point. Additionally, some of these methods assume that multiple characteristics do not intersect, a condition that may not hold in general scenarios.
Furthermore, alternative techniques, such as tensor decomposition \cite{dolgov2021tensor} and polynomial approximation \cite{kalise2018polynomial, kalise2020robust}, have been studied for specific control problems.

Recent advancements in deep learning have given rise to a growing interest in leveraging the extensive representational capabilities of neural networks to solve PDEs \cite{sirignano2018dgm, yu2018deep, raissi2019physics, lu2019deeponet, li2020fourier, yang2023context}. The viscosity solution of HJ PDEs is challenging to obtain directly from the PDE itself, which underscores the development of alternative approaches beyond the established methods like physics-informed neural networks (PINNs) \cite{raissi2019physics}. In response, data-driven methods have been proposed for solving HJ PDEs \cite{nakamura2021adaptive, dolgov2023data, cui2024supervised}; however, these methods face several challenges, including the need for large amounts of training data, the limitation that their performance cannot exceed the accuracy of the numerical methods used to generate the data, and concerns regarding their ability to generalize to unseen scenarios. Moreover, the integration of reinforcement learning techniques to solve HJ PDEs related to control problems has been studied \cite{zhou2021actor, lin2022policy}. Another line of research has focused on the development of specialized neural network architectures that express representation formulas to specific HJ PDEs
\cite{darbon2020overcoming, darbon2021some, darbon2023neural}.
One of the most closely related prior works introduces a deep learning approach for learning implicit solution formulas along with characteristics for scalar conservation laws associated with one-dimensional HJ PDEs \cite{zhang2022implicit}. However, this method does not ensure the attainment of an entropy solution.

This study presents a novel implicit solution formula for HJ PDEs. The proposed implicit formula is derived through the characteristics of the HJ PDE, with the costate identified as the gradient of the solution at the current spatio-temporal point, leading to an implicit representation formula for the solution. We demonstrate that this new formula coincides with the classical Hopf and Lax formulas, which provides the viscosity solution for HJ PDEs in the case where either the Hamiltonian or the initial function is convex (or concave). Notably, the implicit formula is simpler than both the Hopf and Lax formulas, as it does not require the Legendre transform of either the Hamiltonian or the initial function, thereby broadening its practical applicability. Furthermore, although being based on characteristics, the implicit formula alleviates the need to solve the system of characteristic ODEs from the initial state to the present time. From an optimal control perspective, we further explore the connection of the proposed formula with the Pontryagin's maximum principle and Bellman's principle, showing that the proposed implicit solution formula serves as an implicit representation of Bellman’s principle. 
% represents an implicit form of Bellman's principle.

Building on this foundation, we propose a deep learning-based approach to solve HJ PDEs by learning the implicit solution formula. This method approximates the solution as a Lipschitz continuous function, leveraging the powerful expressive capacity of neural networks. Unlike traditional grid-based methods, our approach does not require discretization of the domain, making it highly scalable and efficient, especially for high-dimensional problems. This effectively mitigates the curse of dimensionality, ensuring that computational time and memory usage scale efficiently with dimensionality. Thanks to the inherent simplicity of the implicit solution formula, it obviates the need for computing the Legendre transform and individual characteristic trajectories, thereby enhancing both its applicability and computational efficiency across a wide range of problems. Through extensive and rigorous experimentation, we show that the proposed algorithm provides accurate solutions even for problems with up to 40 dimensions with negligible increases in computational cost. 
Importantly, the method also shows robust performance on various nonconvex HJ PDEs, for which mathematical demonstration has not been established, underscoring its versatility and potential.

We extend our approach to handle HJ PDEs with state-dependent Hamiltonians. In such cases, where the characteristic curves are no longer linear, deriving an implicit solution formula becomes more intricate. To address this, we approximate the characteristic curves as piecewise linear segments over short time intervals, applying the proposed implicit solution formula at each interval. This leads to an efficient time-marching algorithm that can handle state-dependent Hamiltonians, which we validate through a series of experiments  involving high-dimensional optimal control problems. The results demonstrate that the proposed method is not only simple and efficient but also effectively solving a wide range of high-dimensional, nonconvex HJ PDEs, highlighting its potential as a valuable tool for addressing complex optimal control problems and dynamic systems.

\section{Implicit Solution formula of Hamilton-Jacobi Equations}
\subsection{Implicit Solution Formula along Characteristics}\label{sec:implicit_HJ}
In this subsection, we introduce a novel implicit solution formula for the Hamilton-Jacobi partial differential equation (HJ PDE):
\begin{equation}\label{eq:hj}
    \begin{cases}
        u_t+ H\left(\nabla u\right)=0 & \text{ in } \Omega\times(0,T)\\
        u=g & \text{ on } \Omega\times\{t=0\},
    \end{cases}
\end{equation}
where $\Omega\subset\bR^d$ is the spatial domain, $H:\bR^d\rightarrow\bR$ is the Hamiltonian and $g:\Omega\rightarrow\bR$ is the initial function.
System of characteristic ODEs for \eqref{eq:hj},  also known as Hamilton's system, is given by the following:
\begin{subequations}\label{eq:characteristics_all}
\newcommand{\sys}{%
  $\left\lbrace
  \vphantom{\begin{aligned} 1 \\ 1 \\ 1 \\ 1 \end{aligned}}%
  \right.$%
}
\begin{align}
\raisebox{-0.72\height}[0pt][0pt]{\sys}
\dot{\bx}&=\nabla H \label{eq:characteristic_x}\\
\dot{u}&=q + \bp^{\text{T}}\nabla H = -H+\bp^{\text{T}}\nabla H\label{eq:characteristic_u}\\
\dot{q} & = 0\label{eq:characteristic_q}\\
\dot{\bp} & = 0,\label{eq:characteristic_p}
\end{align}
\end{subequations}
where the variables $q$ and $\bp$ are shorthand for the partial derivatives $q=u_t$ and $\bp=\nabla u$, respectively.
From \eqref{eq:characteristic_p} it is clear that the value of $\bp$, which is the sole argument of the Hamiltonian, remains constant along the characteristic. Therefore, the characteristic emanated from $\bx\left(0\right)=\bx_0\in\Omega$ is a straight line
\begin{equation*}
    \bx\left(t\right) = t\nabla H\left(\bp\right)+\bx_0,
\end{equation*}
implying that 
\begin{align*}
   u\left(t,\bx\left(t\right)\right) &= -tH\left(\bp\right) + t\bp^{\text{T}}\nabla H\left(\bp\right) + u\left(\bx_0,0\right)\\
        & = -tH\left(\bp\right) + t\bp^{\text{T}}\nabla H\left(\bp\right) + g\left(\bx_0\right).
\end{align*}
Given the constant nature of $\bp$ along each characteristic line, its value can be determined at any intermediate time between the initial and current times. In this context, we adopt $\bp$ as the gradient of the solution at the current time.
Substituting $\bp=\nabla u$ and expressing $\bx\left(t\right)=\bx\in\Omega$ induces that
\begin{equation*}
\bx_0 = \bx - t\nabla H\left(\nabla u\left(\bx,t\right)\right),
\end{equation*}
and hence 
we attain the following \textbf{implicit solution formula} for HJ PDEs \eqref{eq:hj}:
\begin{equation}\label{eq:implicit_character}
    u\left(\bx,t\right) =-tH\left(\nabla u\right) + t\nabla u^{\text{T}}\nabla H\left(\nabla u\right) + g\left(\bx-t\nabla H\left(\nabla u\right)\right).
\end{equation}
It is worth noting that this implicit formula expresses the solution without requiring the Legendre transform of the Hamiltonian $H$ or the initial function $g$.
Moreover, it does not require to compute individual characteristic trajectories by solving the system of characteristic ODEs. Therefore, it provides a highly practical and straightforward approach to solving HJ PDEs. Building upon this formula, we propose a highly simple and effective deep learning-based methodology for solving HJ PDEs in \cref{sec:method}.

A key distinction between conventional approaches based on characteristics and the proposed implicit solution formula lies in the treatment of $\bp$, which is chosen as the gradient of the solution $\nabla u$ at the current time $t$. Since $\bp$ remains constant along each characteristic line, it can be readily determined from the initial data. Consequently, conventional methods typically express $\bp$ in terms of $\nabla g\left(\bx_0\right)$. However, these approaches are limited in situations where no characteristic traces back to the initial time $t=0$, resulting in the gradient at the current time not being accessible from the initial data. In contrast, our approach employs the current value of $\bp\left(t\right)=\nabla u\left(\bx,t\right)$, allowing the implicit solution formula to effectively handle such scenarios.

It is well-established that under certain assumptions on the Hamiltonian $H$ and the initial function $g$, a representation formula for the viscosity solution can be derived. The first is the \textit{Hopf-Lax formula} 
\begin{equation}\label{eq:hopf_lax}
     u\left(\bx,t\right) = \inf_{\by} \Bigl\{ tH^\ast\left(\frac{\bx-\by}{t}\right) + g\left(\by\right)\Bigr\},
\end{equation}
which holds for convex (or concave) $H$ and Lipschitz $g$ \cite{hopf1965generalized, bardi1984hopf, lions1982generalized}, or for Lipschitz and convex $H$ and continuous $g$ \cite{subbotin2013generalized}, or also for
strictly convex $H$ and lower semicontinuous (l.s.c.) $g$  \cite{kruzhkov1960cauchy,kruzhkov1967generalized}.
Here, where $H^\ast$ is the Legendre transforms of $H$.
On the other hand, \textit{Hopf formula}
\begin{equation}\label{eq:hopf}
    u\left(\bx,t\right) = -\inf_{\bz} \Bigl\{ g^\ast\left(\bz\right) + tH\left(\bz\right) - \bx^{\text{T}}\bz \Bigr\}
\end{equation}
is valid for Lipschitz and convex (or concave) $g$ and merely continuous
$H$ \cite{hopf1965generalized, bardi1984hopf}, or for convex $g$ and Lipschitz $H$ \cite{subbotin2013generalized}. In the following, we demonstrate that the proposed implicit solution formula \eqref{eq:implicit_character} represents these two respective formulas under the conditions under which they hold.

\begin{theorem}\label{thm:implicit_formula_HL}
    Assume the Hamiltonian $H$ is differentiable and satisfies
\begin{equation}\label{eq:assump_H}
    \begin{cases}
        \bp\mapsto H\left(\bp\right) \text{ is strictly convex or concave},\\
        \lim_{\left\vert\bp\right\vert\rightarrow\infty}\frac{H\left(\bp\right)}{\left\vert\bp\right\vert}=+\infty,
    \end{cases}
\end{equation}
and the initial function $g$ is l.s.c.
Then, the continuous function $u$ that satisfies the implicit solution formula \eqref{eq:implicit_character} is the viscosity solution of \eqref{eq:hj} a.e.
\end{theorem}
\begin{proof}
First, we can observe that the implicit solution formula \eqref{eq:implicit_character} exactly satisfies the initial condition $u=g$ of \eqref{eq:hj} at the initial time $t=0$.

Under the assumptions, the viscosity solution of the HJ PDE is described by the Hopf-Lax formula \eqref{eq:hopf_lax}.
By expanding the Legendre transform in the Hopf-Lax formula \eqref{eq:hj}, the viscosity solution is expressed as follows
    \begin{align}
        u\left(\bx,t\right) & = \inf_{\by} \sup_{\bz} \Bigl\{t\Bigl(\bz^{\text{T}}\left(\frac{\bx-\by}{t}\right)-H\left(\bz\right) \Bigr) + g\left(\by\right) \Bigr\}\\
         & = \inf_{\by} \sup_{\bz} \Bigl\{\bz^{\text{T}}\left(\bx-\by\right)-tH\left(\bz\right) + g\left(\by\right) \Bigr\} \label{eq:inf_sup},
    \end{align}
 % and $g$, respectively.
The Euler-Lagrange equation of Hopf-Lax formula leads to
\begin{equation}\label{eq:optimal_hl}
\by^\star =    \underset{\by}{\text{argmin}} \Bigl\{ tH^\ast\left(\frac{\bx-\by}{t}\right) + g\left(\by\right)\Bigr\} = \bx-t\nabla H\left(\nabla u\right).
\end{equation}
Furthermore, differentiating \eqref{eq:inf_sup} with respect to $\bz$ provides that the optimal $\bz^\star$ 
satisfies 
\begin{equation*}
    \bx-\by^\star - t\nabla H\left(\bz^\star\right)=0.
\end{equation*}
Together with \eqref{eq:optimal_hl}, we have
\begin{equation}\label{eq:optimal_h}
    \bz^\star = \nabla u.
\end{equation}
Substituting these \eqref{eq:optimal_hl} and \eqref{eq:optimal_h} into \eqref{eq:inf_sup} results in the implicit formula \eqref{eq:implicit_character}.
\end{proof}

\begin{theorem}\label{thm:implicit_formula_Hopf}
    Assume the initial function $g$ satisfies
\begin{equation*}
    \begin{cases}
        \bx\mapsto g\left(\bx\right) \text{ is convex or concave},\\
        \lim_{\left\vert\bx\right\vert\rightarrow\infty}\frac{g\left(\bx\right)}{\left\vert\bx\right\vert}=+\infty,
    \end{cases}
\end{equation*}
that the Hamiltonian $H$ is continuous, and that either the $H$ or $g$ is Lipschitz continuous.
Then, the continuous function $u$ that satisfies the implicit solution formula \eqref{eq:implicit_character} is the viscosity solution of \eqref{eq:hj} a.e.
\end{theorem}
\begin{proof}
Since the viscosity solution is described by the Hopf formula \eqref{eq:hopf_lax} under these assumptions, it can be written as follows:
    \begin{align}
        u\left(\bx,t\right) & = -g^\ast\left(\bz^\star\right) - tH\left(\bz^\star\right) + \bx^{\text{T}}\bz^\star\label{eq:hopf_explicit}\\
         & = \inf_{\by} \Bigl\{\bz^{\star\text{T}}\left(\bx-\by\right)-tH\left(\bz^\star\right) + g\left(\by\right) \Bigr\} \label{eq:inf_sup_hopf}\\
         & = \bz^{\star\text{T}}\left(\bx-\by^\star\right)-tH\left(\bz^\star\right) + g\left(\by^\star\right).
    \end{align}
Differentiating the both side of \eqref{eq:hopf_explicit} with respect to $\bx$ induces
\begin{equation*}
    \frac{\partial u}{\partial\bx} = -\frac{\partial}{\partial\bz}\Bigl\{ g^\ast\left(\bz^\star\right) + tH\left(\bz^\star\right)\Bigr\}\cdot \frac{\partial\bz^\star}{\partial\bx} + \bz^\star = \bz^\star,
\end{equation*}
where the last equality follows from the optimality of $\bz^\star$.
Consequently, we have $\bz^\star = \nabla u$.
Furthermore, differentiating \eqref{eq:inf_sup_hopf} with respect to $\bz$ provides that the optimal $\by^\star$ 
satisfies 
\begin{equation*}
    \bx-\by^\star - t\nabla H\left(\bz^\star\right)=0,
\end{equation*}
that is,
\begin{equation*}
\by^\star = \bx - t\nabla H\left(\bz^\star\right) = \bx - t\nabla H\left(\nabla u\right),
\end{equation*}
which concludes the proof.
\end{proof}
Theorems \ref{thm:implicit_formula_HL} and \ref{thm:implicit_formula_Hopf} offers the theoretical validation of the implicit solution formula \eqref{eq:implicit_character} under the assumption of convexity of the Hamiltonian $H$ or the initial function $g$. However, this result has not yet been extended to the nonconvex case. Nonetheless, as illustrated in Sub\cref{sec:exp_nonconvex}, we present robust empirical evidence demonstrating the performance of the proposed approach through extensive experiments on a diverse range of nonconvex examples, where neither the Hamiltonian nor the initial function is convex (or concave). These results suggest the potential applicability and validity of the proposed formula in such scenarios.

To facilitate comprehension of the implicit solution formula, a simple example is presented.
\begin{example}\label{ex:main}
    Consider a one-dimensional example with a quadratic Hamiltonian and a homogeneous initial condition:
    \begin{equation}\label{eq:ex_hj}
    \begin{cases}
        u_t+ u_x^2=0 & \text{ in } \bR\times(0,\infty)\\
        u=0 & \text{ on } \R\times\{t=0\}.
    \end{cases}
\end{equation}
The viscosity solution to this problem is $u^\ast = 0$.
Note that there are infinitely many Lipschitz functions satisfying \eqref{eq:ex_hj} a.e. \cite{evans2022partial}, for instance,
\begin{equation*}
    v\left(x,t\right)=\begin{cases}
    0 & \text{if }\mid x\mid\geq t\\
    x-t & \text{if }0\leq x\leq t\\
    x-t & \text{if }-t\leq x\leq 0.
    \end{cases}
\end{equation*}
This example shows that, although there are infinitely many Lipschitz functions that satisfy the HJ PDE, the implicit solution formula uniquely characterizes the viscosity solution.
The implicit solution formula \eqref{eq:implicit_character} corresponding to \eqref{eq:ex_hj} is written as
\begin{equation}\label{eq:ex_implicit_form}
    u = tu_x.
\end{equation}
For $t=0$, \eqref{eq:ex_implicit_form} satisfies the initial condition $u=0$.
For a fixed time $t>0$, \eqref{eq:ex_implicit_form} represents an ordinary differential equation (ODE) with respect to the variable $x$ with a coefficient that depends on $t$, and the ODE admits infinitely many solutions
\begin{equation*}
    u = Ce^{tx},\ \forall C\in\bR.
\end{equation*}
However, in order to satisfy the initial condition $u=0$ at $t=0$, it follows that $C$ must be zero.
Therefore, the viscosity solution $u^\ast=0$ is the unique continuous function that satisfies the implicit solution formula \eqref{eq:ex_implicit_form}.
\end{example}

This example illustrates that, despite the existence of an infinitely many weak solutions to the governing HJ PDE, the continuous function that satisfies the implicit solution formula \eqref{eq:implicit_character} is the unique viscosity solution. However, it also suggests that, at a fixed time $t>0$, the implicit formula may admit multiple solutions. For a fixed $t>0$, the implicit solution formula \eqref{eq:implicit_character} describes a first-order nonlinear static PDE (or an ODE in the one-dimensional case) in $\bx$, where the time variable $t$ appears as coefficients. The absence of boundary conditions in this static PDE at fixed $t>0$ naturally leads to the ill-posedness of the PDE with multiple solutions. Therefore, the condition that the implicit formula \eqref{eq:implicit_character} satisfies the initial condition at $t=0$ is crucial, and finding a continuous function that satisfies the implicit solution formula across the entire spacetime domain from the initial data is essential for obtaining the unique viscosity solution. It is noteworthy, however, that the above example is taken in the unbounded spatial domain $\bR$. For the general case of HJ PDEs on a bounded domain $\Omega$, boundary conditions are specified. In such cases, the given boundary condition serves as the boundary condition for the static PDE described by \eqref{eq:implicit_character} at a fixed time, thereby ensuring the uniqueness of the solution.

\begin{remark} (Level set propagation)
If the Hamiltonian $H$ is homogeneous of degree one in its gradient argument, i.e., $H$ takes of the form with a function $\bv:\bR^d\rightarrow\bR^d$ 
\begin{equation}\label{eq:normal_velocity}
    H\left(\nabla u\right) = \bv\left(\frac{\nabla u}{\left\Vert \nabla u\right\Vert}\right)^{\text{T}}\nabla u,
\end{equation}
then the implicit formula \eqref{eq:implicit_character} comes down to
the following simple formula a.e. 
\begin{equation}\label{eq:implicit_lvset}
    u\left(\bx,t\right) = g\left(\bx-t\bv\left(\frac{\nabla u}{\left\Vert \nabla u\right\Vert}\right)\right),
\end{equation}
where the solution $u$ is constant along the characteristics.
\end{remark}

\subsection{Control Perspectives on the Implicit Solution Formula}
In this subsection, we revisit the implicit solution formula \eqref{eq:implicit_character} from the perspective of control theory, elucidating that it represents an implicit formulation of Bellman's principle. This perspective also enables a comprehensive exploration of the relationships between the implicit solution formula and Pontryagin's maximum principle (PMP) and the Hopf-Lax formula \eqref{eq:hopf_lax}, while also highlighting the distinctions between these established approaches and the proposed implicit formula.

Let $L=L\left(\bq\right):\bR^d\rightarrow\bR$ be the corresponding Lagrangian, that is, $L=H^\ast$, the Legendre transform of $H$. Under the assumption \eqref{eq:assump_H} on the Hamiltonian $H$, the Lagrangian satisfies
\begin{equation*}
    \begin{cases}
        \bq\mapsto L\left(\bq\right) \text{ is convex},\\
        \lim_{\left\vert\bq\right\vert\rightarrow\infty}\frac{L\left(\bq\right)}{\left\vert\bq\right\vert}=+\infty,
    \end{cases}
\end{equation*}
and $H\left(\bp\right)=L^\ast\left(\bp\right)=\underset{\bq}{\sup}\left\{\bp^{\text{T}}\bq - L\left(\bq\right)\right\}$.

It is well-known that the viscosity solution $u$ of the HJ PDE 
\begin{equation}\label{eq:oc_hjb}
   \begin{cases} u_t+H\left(\nabla u\right)=u_t+\underset{\bq}{\sup}\left\{\nabla u^{\text{T}}\bq - L\left(\bq\right)\right\}=0,\\
   u\left(\bx,0\right)=g\left(\bx\right)
   \end{cases}
\end{equation}
is represented by the value function of the following corresponding optimal control problem:
\begin{equation}\label{eq:oc_value_ft}
     u\left(\bx,t\right)=\underset{\bq}{\inf} \left\{\int_0^t L\left(\bq\left(s\right)\right)\diff s+g\left(\by\left(0\right)\right) : \by\left(t\right)=\bx, \dot{\by}(s)=\bq(s), 0\leq s\leq t \right\}.
\end{equation}

Pontryagin's maximum principle (PMP) states that the optimal trajectory of state $\by\left(t\right)$ arriving at $\by\left(t\right)=\bx$ and costate $\bp\left(t\right)$ satisfies
\begin{subequations}\label{eq:pmp}
\newcommand{\sys}{%
  $\left\lbrace
  \vphantom{\begin{aligned} 1 \\ 1 \\ 1 \end{aligned}}%
  \right.$%
}
\begin{align}
\raisebox{-0.72\height}[0pt][0pt]{\sys}
 \dot{\by} & = \bq,\ \by\left(t\right)=\bx,\label{eq:oc_characteristic_x}\\
\dot{\bp} & = 0,\ \bp\left(0\right)=\nabla g\left(\by(0)\right), \label{eq:oc_characteristic_p}\\
 \bq & = \underset{\bv}{\text{argmax}}\left\{\bp^{\text{T}}\bv - L\left(\bv\right)\right\}\label{eq:oc_optimal_a}.
\end{align}
\end{subequations}
Note that this is identical to the characteristic ODEs for the state $\bx$ \eqref{eq:characteristic_x} and the gradient of the solution \eqref{eq:characteristic_p} of the HJ PDEs. Therefore, PMP implies that the characteristic of the HJ PDE corresponds to the optimal trajectory.

We now establish that both the Hopf-Lax formula and the implicit solution formula can be derived from the PMP in conjunction with the definition of the value function. This facilitates a comprehensive understanding of their relationships and differences.

\paragraph{\textbf{Hopf-Lax Formula}}
Since the costate $\bp$ is constant along the optimal trajectory \eqref{eq:oc_characteristic_p}, the last condition \eqref{eq:oc_optimal_a} of the PMP  implies that $\bq$ is also constant. Therefore, from \eqref{eq:oc_characteristic_x}, it follows that the optimal trajectory of $\by$ is a straight line, whose solution is given by
\begin{equation}\label{eq:oc_pull_y}
    \by\left(0\right) = \by\left(t\right)-t\bq = \bx-t\bq.
\end{equation}
Therefore, the optimal $\bq$ is expressed by $\by\left(0\right)$ as follows:
\begin{equation}\label{eq:linear_optimal_trajectory}
    \bq = \frac{\bx-\by\left(0\right)}{t},
\end{equation}
and hence, the minimization with respect to $\bq$ can be transformed into a minimization with respect to $\by\left(0\right)\in\bR^d$. Substituting this relation \eqref{eq:linear_optimal_trajectory} into the definition of the value function \eqref{eq:oc_value_ft} leads to the following Hopf-Lax formula:
\begin{align*}
    u\left(\bx,t\right)&=\underset{\by\in\bR^d}{\inf} \left\{\int_0^tL\left(\frac{\bx-\by}{t}\right)\diff s+ g\left(\by\right)\right\}\\
    & = \underset{\by\in\bR^d}{\inf}\left\{tL\left(\frac{\bx-\by}{t}\right)+g\left(\by\right)\right\}\\
    & = \underset{\by\in\bR^d}{\inf} \left\{tH^\ast\left(\frac{\bx-\by}{t}\right)+g\left(\by\right)\right\}.
\end{align*}
In other words, the Hopf-Lax formula is derived by substituting the control $\bq$ in terms of the initial state $\by\left(0\right)=\by$, leveraging the fact that the optimal trajectory is linear \eqref{eq:linear_optimal_trajectory}, as determined by the characteristic ODE of the PMP.

\paragraph{\textbf{Implicit Solution formula}}
The implicit solution formula \eqref{eq:implicit_character} is derived in a manner analogous to the Hopf-Lax formula, but it differs by expressing $\bp$ as the gradient of the value function $\nabla u$ and additionally removing the Legendre transform.
From the optimality of $\bq$ in \eqref{eq:oc_optimal_a}, the Hamiltonian is written as
\begin{equation}\label{eq:oc_hamiltonian_form}
    H\left(\bp\right) = \bp^{\text{T}}\bq + L\left(\bq\right).
\end{equation}
It follows that 
\begin{equation}\label{eq:oc_diff_H}
    \nabla_{\bp}H = \frac{\partial H}{\partial\bp} + \frac{\partial H}{\partial\bq}\frac{\partial\bq}{\partial\bp} = \frac{\partial H}{\partial\bp}=\bq,
\end{equation}
where $\frac{\partial H}{\partial\bq}=0$ is induced from \eqref{eq:oc_optimal_a}.
Let $\bq^\ast$ be the optimal control. Putting \eqref{eq:oc_optimal_a}, \eqref{eq:oc_pull_y}, and \eqref{eq:linear_optimal_trajectory} to the definition of the value function in \eqref{eq:oc_value_ft} leads to the following formula of the value function:
\begin{align}\label{eq:oc_implicit_formula}
    \begin{split}    u\left(\bx,t\right)&=tL\left(\bq^\ast\right)+g\left(\bx-t\bq^\ast\right)\\
        & = t\left(H\left(\bp\right)-\bp^{\text{T}}\bq^\ast\right)+g\left(\bx-t\bq^\ast\right)\\
        & = t\left(H\left(\bp\right)-\bp^{\text{T}}\nabla_{\bp}H\left(\bp\right)\right)+g\left(\bx-t\nabla_{\bp}H\left(\bp\right)\right),
    \end{split}
\end{align}
where the second and third equalities are derived from \eqref{eq:oc_hamiltonian_form} and \eqref{eq:oc_diff_H}, respectively. In other words, by substituting these two expressions \eqref{eq:oc_hamiltonian_form} and \eqref{eq:oc_diff_H}, we derive the formula for the value function $u$ that is independent of both the control $\bq$ and the Legendre transform.
Since the optimal $\bp$ is the gradient of the value function $\nabla u$, the solution formula \eqref{eq:oc_implicit_formula} derived from PMP is identical to the implicit solution formula \eqref{eq:implicit_character}.
Furthermore, it is important to note that the definition of the value function \eqref{eq:oc_value_ft} precisely encapsulates the integral of the characteristic ODE of $u$ \eqref{eq:characteristic_u}; that is, it directly represents the solution to the characteristic ODE of $u$ \eqref{eq:characteristic_u}. In other words, the characteristic ODE of $u$ \eqref{eq:characteristic_u}, which is not explicitly included in the PMP formula \eqref{eq:pmp}, is inherently embedded within the construction of Bellman's value function.

Consequently, the PMP \eqref{eq:pmp}, the Bellman's value function \eqref{eq:oc_value_ft}, the Hopf-Lax formula \eqref{eq:hopf_lax}, and the proposed implicit solution formula \eqref{eq:implicit_character} are all interconnected. The PMP states that the characteristics of the HJ PDE correspond to the optimal trajectory, and Bellman's principle expresses the value function in terms of the solution to the characteristic ODE of $u$ \eqref{eq:characteristic_u}. Together, they imply that the viscosity solution to the HJ PDE \eqref{eq:hj} is defined along the characteristics. 

However, there are notable differences in how these formulas yield the solution to \eqref{eq:hj} from a practical perspective.
The PMP necessitates the solution of a single trajectory of the characteristic ODEs, which implies that, when attempting to compute the value function, one must solve a system of ODEs for each trajectory, introducing significant computational complexity. The Hopf-Lax formula, by exploiting the linearity of the optimal trajectory, eliminates the need to solve such ODEs. However, it involves the computation of the Legendre transform of the Hamiltonian $H$, ultimately leading to a challenging min-max problem. In contrast, the implicit solution formula \eqref{eq:implicit_character} alleviates both the ODE solving of PMP and the min-max problem from the Hopf-Lax formula by leveraging the fact that the optimal costate $\bp$ is the gradient of the solution $\nabla u$. Consequently, compared to the PMP and the Hopf-Lax formula, the proposed implicit formula provides a more practical and widely applicable approach for solving HJ PDEs.

\section{Learning Implicit Solution with Neural Networks}\label{sec:method} % Neural Implicit Solution for HJ PDEs
In this section, we introduce a deep learning-based approach for solving the implicit solution formula \eqref{eq:implicit_character}. Building upon the implicit solution formula, we propose the following minimization problem:
\begin{equation}\label{eq:implicit_loss}
    \min_{u} \cL\left(u\right)\coloneqq \int_0^T\int_\Omega \Bigl(u+tH\left(\nabla u\right)-t\nabla u^{\text{T}}\nabla H\left(\nabla u\right)-g\left(\bx-t\nabla H\left(\nabla u\right)\right)\Bigr)^2 \diff\bx \diff t.
\end{equation}
The minimization problem \eqref{eq:implicit_loss} is inherently complex to be efficiently solved using classical numerical methods. To address this challenge, we propose a deep learning framework that has shown significant effectiveness in optimizing complex problems. This approach enables the scalable learning of the implicit solution formula, even in high-dimensional settings, thereby allowing the solution of the HJ PDEs \eqref{eq:hj} to be represented by a neural network, which is a Lipschitz function.

\subsection{Implicit Neural Representation}\label{sec:inr}
We represent the solution $u$ of the HJ PDE \eqref{eq:hj} using a standard artificial neural network architecture, a multi-layer perceptron (MLP). The MLP is a function $u_\theta:\bR^d\times\bR\rightarrow\bR$ defined as the composition of functions, which can be expressed as follows:
\begin{equation}\label{eq:MLP}
u_\theta\left(\bx,t\right)= W\left(h_L\circ \cdots \circ h_0\left(\bx,t\right)\right) + \bb,\ \left(\bx,t\right)\in\bR^n\times\bR,
\end{equation}
where $L\in\bN$ is a given depth, $W\in\bR^{1\times d_L}$ is a weight of the output layer, $\bb\in \bR$ is an output bias and the perceptron (also known as the hidden layer) $h_{\ell}:\bR^{d_{\ell-1}} \rightarrow \bR^{d_{\ell}}$ is defined by 
\begin{equation*}
h_{\ell}\left(\by\right)=\sigma\left(W_\ell \by + \bb_\ell\right),\ \by\in\bR^{d_{\ell-1}},\ \text{for all } \ell=0,\dots,L,
\end{equation*}
for weight matrices $W_\ell\in \bR^{d_{\ell}\times d_{\ell-1}}$ with the input dimension $d_{-1}=d+1$, bias vectors $\bb_\ell\in\bR^{d_\ell}$, and a non-linear activation function $\sigma$. The dimensions $d_\ell$ of the hidden layers are also called by the width of the network. A shorthand notation $\theta$ is used to refer to all the parameters of the network, including the weights $\left\{W, W_0,\cdots, W_L\right\}$ and biases $\left\{\bb, \bb_0,\cdots ,\bb_L\right\}$.
Since Lipschitz continuous activation functions $\sigma$ are used, the MLP $f_\theta$ is a Lipschitz function and is also bounded on a bounded domain.
Given the current parameter configuration, the parameters $\theta$ are successively adapted by minimizing an assigned loss function explained in the subsequent section.

Representing the solution of HJ PDEs using neural networks offers a scalable and efficient approach for modeling the spatio-temporal dependencies of the solution, offering several advantages over classical numerical schemes. Classical methods discretize the spatial vector field using primitives such as meshes, which scale poorly with the number of spatial samples. In contrast, representing the spatio-temporal function through networks known as implicit neural representations (INRs) encodes spatial and temporal dependencies through neural network parameters $\theta$, each globally influencing the function. 
Consequently, the memory usage of INRs remains independent of the spatial sample size, being determined solely by the number of network parameters, which enables scalability in high-dimensional settings as evidenced in \cref{tab:mse} for the proposed method. Additionally, INRs are adaptive, leveraging their capacity to represent arbitrary spatio-temporal locations of interest without requiring memory expansion or structural modifications. The expressivity of non-linear neural networks enables INRs to achieve superior accuracy compared to mesh-based and meshless methods, even under the same memory constraints \cite{takikawa2021neural, chen2023implicit}. Furthermore, INRs represent the solution as a continuous function rather than at discrete points, with activation functions that can be tailored to the solution's regularity. Thanks to the architecture of MLPs, exact derivatives can be computed via the chain rule, eliminating the need for numerical differentiation methods such as finite differences. The partial derivatives of $u_\theta$ are efficiently computed using automatic differentiation library (\texttt{autograd}) \cite{paszke2017automatic}.

\subsection{Training}
Given that the solution $u$ is represented by the neural network $u_\theta$, the minimization problem \eqref{eq:implicit_loss} reduces to finding the network parameters $\theta$ that minimize $\cL$ in \eqref{eq:implicit_loss}. 
For notational convenience, we denote
\begin{equation*}
\cS\left(u\right) = u+tH\left(\nabla u\right)-t\nabla u^{\text{T}}\nabla H\left(\nabla u\right)-g\left(\bx-t\nabla H\left(\nabla u\right)\right).
\end{equation*}
The integral of $\cL$ is approximated using Monte Carlo methods
\begin{equation}\label{eq:discrete_loss}
    \hat{\cL}\left(\theta\right)=\frac{1}{M}\sum_{j=1}^{M}\cS\left(u_\theta\left(\bx_j,t_j\right)\right)^2,
\end{equation}
with the $M$ collocation points $\left\{\left(\bx_j,t_j\right)\right\}_{j=1}^{M}$ randomly sampled from a uniform distribution $U\left(\Omega\times\left[0,T\right]\right)$. This empirical loss $\hat{\cL}$ serves as the loss function for training the neural network.
The current network parameters are updated using a gradient-based optimizer to minimize the loss function $\hat{\cL}$. 
During training, different random collocation points are employed in each iteration to ensure accurate learning across the entire domain.
The partial derivatives of the network $u_\theta$ are computed through \texttt{autograd} when calculating the loss. The training procedure for optimization using gradient descent is summarized in \cref{alg:GD}.

\begin{algorithm}
\caption{Algorithm for Learning Implicit Solution of HJ PDEs}\label{alg:GD}
\begin{algorithmic}[1]
    \STATE{Initialize the network $u_\theta$ with an initial network parameter $\theta_0$.}
    \FOR{$n=0,\cdots,N$}
        \STATE{Randomly sample $M$ collocations points $\left\{\left(\bx_j,t_j\right)\right\}_{j=1}^{M}\sim U\left(\Omega\times\left[0,T\right]\right)$.}
        \STATE{Calculate the loss by Monte Carlo integration
        \[
        \hat{\cL}\left(\theta_{n}\right)=\frac{1}{M}\sum_{j=1}^{M}\cS\left(u_{\theta_n}\left(\bx_j,t_j\right)\right)^2.
        \]}
        \STATE{Update $\theta_n$ by gradient descent with a step size $\alpha>0$
        \[
        \theta_{n+1} \leftarrow \theta_n - \alpha\nabla_\theta \hat{\cL}\left(\theta_{n}\right).
        \]}
        \ENDFOR
   \RETURN $u_{\theta_N}$ as the predicted viscosity solution to the HJ PDE \eqref{eq:hj}.
\end{algorithmic}
\end{algorithm}

This algorithm is considerably simpler than existing methodologies in several respects and yields remarkable results, as demonstrated in \cref{sec:experiments}. Previous approaches \cite{darbon2016algorithms, chow2017algorithm, chow2019algorithm, chen2024hopf}, which aimed to obtain the viscosity solution via the Hopf or Lax formula, involved calculating the Legendre transform of the Hamiltonian or the initial function. Therefore, these methods were restricted to problems where the Legendre transform was easily computable or required solving numerically intensive min-max problems for each spatio-temporal point, limiting their general applicability. In contrast, our approach bypasses the Legendre transform by using an implicit formula, enabling us to handle a broader class of Hamiltonians and initial functions. Moreover, while prior methods based on characteristics or PMP \cite{kang2015causality, kang2017mitigating, yegorov2021perspectives} necessitated solving a system of ODEs to track individual trajectories, the proposed method eliminates the requirement for explicit trajectory computation.

The proposed method also overcomes the limitations of classical grid-based numerical methods, which face challenges when dealing with high-dimensional or large-scale problems due to the increasing number of grid points required as the dimension or domain size grows.
Unlike classical methods, the deep learning approach is characterized by its mesh-free nature, which precludes the necessity for a grid discretization of the computational domain.  The mesh-free approach allows for the random selection of collocation points in each iteration, with the selected points gradually covering the domain as iterations proceed.
Consequently, the computational and memory requirements do not increase significantly with higher dimensions, as evidenced in \cref{tab:mse} for the proposed method. Furthermore, under certain mild assumptions, it has been demonstrated that this stochastic gradient descent applied using randomly sampled collocation points converges to the minimizer of the original expectation loss \cite{karandikar2024convergence}. The absence of mesh generation also simplifies the practical implementation of the method.

This approach also offers distinct advantages over existing deep learning methods for solving PDEs. As an unsupervised learning method, it solves HJ PDEs given the Hamiltonian and initial condition, without requiring solution data for training. This addresses the limitations of supervised learning methods \cite{nakamura2021adaptive, dolgov2023data, cui2024supervised}, which rely on extensive solution data and do not guarantee generalization to unseen problems. The proposed approach also offers strengths compared to the established unsupervised methods, such as PINNs \cite{raissi2019physics} and the DeepRitz method \cite{yu2018deep}. DeepRitz, which is based on a variational formulation, is not suitable for HJ PDEs. PINNs, on the other hand, use the residual of the PDE itself as the loss function, which cannot guarantee obtaining the viscosity solution for HJ PDEs among multiple solutions. Since the viscosity solution cannot be directly derived from the PDE itself, there are inherent challenges in obtaining it from the PDE residual loss used in PINNs. In comparison, the proposed method learns an implicit formula for the solution that naturally inherits the properties of the viscosity solution through the characteristic equation, enabling effective solutions to HJ PDEs.
Furthermore, most deep learning methods for solving PDEs, including PINNs and DeepRitz method, use a training loss function that is the linear sum of the loss term corresponding to the PDE and the loss term for the initial condition. This requires a regularization parameter to balance the two loss terms, which is highly sensitive and difficult to optimize \cite{wang2022and}. In contrast, the proposed method employs an implicit solution formula, whereby the initial condition is automatically incorporated by substituting $t=0$ into \eqref{eq:discrete_loss}. As a result, our approach eliminates the need for a regularization parameter, using only a single loss function and obviating the distinction between the initial condition and the PDE.

When boundary conditions are specified in the spatial domain $\Omega$, we incorporate additional loss terms to enforce these conditions. For instance, when a Dirichlet boundary condition is imposed with the boundary function $h:\partial\Omega\rightarrow[0,T]\rightarrow\bR$, the following loss function is used:
\begin{equation*}
    \frac{1}{M_b} \sum_{j=1}^{M_b} \left(u\left(\bx_j^b,t^b_j\right)-h\left(\bx^b_j,t^b_j\right)\right)^2,
\end{equation*}
where the $M_b$ boundary collocation points $\left(\bx^b_j,t^b_j\right)\in\partial\Omega\times\left[0,T\right]$ are randomly sampled from a uniform distribution.
Similarly, for a periodic boundary condition, the boundary loss term is given as follows:
\begin{equation*}
     \frac{1}{M_b} \sum_{j=1}^{M_b} \left(u\left(\bx^b_j,t^b_j\right)-u\left(\by^b_j,t^b_j\right)\right)^2,
\end{equation*}
where $\by^b_i\in\partial\Omega$ represents the point on the opposite side of the domain $\Omega$ corresponding to $\bx^b_i$.
This boundary loss, weighted by the regularization parameter $\lambda>0$, is then added to the implicit solution loss $\hat{\cL}$ \eqref{eq:discrete_loss} to form the total training loss.

\begin{remark}
   If the goal is to obtain a solution at a specific time $t=T$ rather than over the entire temporal evolution,  integrating over time in the loss function may not be necessary. However, as shown in \cref{ex:main} in \cref{sec:implicit_HJ}, when the PDE lacks boundary conditions, at a fixed time $t>0$, the the implicit solution formula at a fixed $t$ results in a differential equation without boundary conditions, leading to the possibility of multiple spurious solutions. To address this, it is preferable to incorporate an integral over the entire temporal domain in the loss function, thereby training a continuous network to find a continuous solution that satisfies \eqref{eq:implicit_character} across the entire spacetime domain. On the other hand, when boundary conditions are given in the HJ PDE \eqref{eq:hj}, these can serve as the boundary condition for the differential equation \eqref{eq:implicit_character} at the fixed time, ensuring the uniqueness of the solution. In such cases, training the model exclusively with respect to the terminal time $t=T$ may suffice.
\end{remark}

\subsection{State-dependent Hamiltonian}\label{sec:state_H}
In this subsection, we propose an algorithm for the case of a state-dependent Hamiltonian, inspired by the implicit solution formula \eqref{eq:implicit_character}.
Consider the state-dependent HJ PDE defined in a domain $\Omega\subset\bR^d$
\begin{equation}\label{eq:x_dependent_hj}
    \begin{cases}
        u_t+ H\left(\bx, \nabla u\right)=0 & \text{ in } \Omega\times(0,T)\\
        u=g & \text{ on } \Omega\times\{t=0\}.
    \end{cases}
\end{equation}
The system of characteristic ODEs of \eqref{eq:x_dependent_hj} is given by
\begin{subequations}\label{eq:pmp}
\newcommand{\sys}{%
  $\left\lbrace
  \vphantom{\begin{aligned} 1 \\ 1 \\ 1 \end{aligned}}%
  \right.$%
}
\begin{align}
\raisebox{-0.72\height}[0pt][0pt]{\sys}
\dot{\bx} &=\nabla_\bp H\label{eq:x_dependent_characteristic_x}\\
\dot{u}&=-H+\bp^{\text{T}}\nabla_{\bp} H\label{eq:x_dependent_characteristic_u}\\
\dot{\bp} &= - \nabla_{\bx}H,\label{eq:x_dependent_characteristic_p}
\end{align}
\end{subequations}
where $\bp=\nabla u$.
Given that $\bp$ is no longer a constant along the characteristic \eqref{eq:x_dependent_characteristic_p}, the characteristics are not linear but instead curves. Consequently, computing the integral along these curves becomes highly challenging, making the derivation of an implicit solution formulation difficult.

\paragraph{Piecewise Linear Approximation of Characteristic Curves} We assume that $\bp$ remains relatively constant over short time intervals.
% , and therefore, we treat $\bp$ as constant along the characteristic within each short time interval. 
In other words, we approximate the characteristic curve as linear over short time segments.
To this end, we discretize the temporal domain by
\begin{equation*}
    t_0=0<t_1=\Delta t<t_2=2\Delta t<\cdots,t_N=N\Delta t=T.
\end{equation*}
For each $k=0,\cdots,N-1$, we can write the solution as follows: for $t\in\left[t_k,t_k+\Delta t\right]$,
\begin{equation*}
    u\left(\bx,t\right) = u\left(\bx,t_k+\tau\right) = u^k\left(\bx,\tau\right)
\end{equation*}
with $t=t_k+\tau$, $\tau\in\left[0,\Delta t\right]$. Then $u^k$ can be regarded as the solution of the following HJ PDE for time $0\leq t\leq \Delta t$ with the initial function $u^k\left(\cdot,0\right)=u^{k-1}\left(\cdot,\Delta t\right) = u\left(\cdot, k\Delta t\right)$:
\begin{equation}\label{eq:x_dependent_hj}
    \begin{cases}
        u^k_t+ H\left(\bx, \nabla u^k\right)=0 & \text{ in } \Omega\times(0,\Delta t)\\
        u^k\left(\bx,0\right)=u^{k-1}\left(\bx,\Delta t\right) & \text{ on } \Omega.
    \end{cases}
\end{equation}
Assuming that $\bp$ remains constant within each short time interval $\left[t_k,t_k+\Delta t\right]$, similar to the state-independent Hamiltonian discussed in \cref{sec:implicit_HJ}, we can derive the following \textit{implicit solution formula} for \eqref{eq:x_dependent_hj}:
\begin{align}\label{eq:x_dependent_implicit_formula}
    u^k\left(\bx,\tau\right) & = -\tau H\left(\bx,\nabla u^k\left(\bx,\tau\right)\right) + \tau\nabla u^k\left(\bx,\tau\right)^{\text{T}}\nabla_\bp H\left(\bx,\nabla u^k\left(\bx,\tau\right)\right)\\
    &+ u^{k-1}\left(\bx-\tau\nabla_\bp H\left(\bx,\nabla u^k\right), \Delta t\right).
\end{align}
This can be regarded as an implicit Euler discretization of the characteristic ODE \eqref{eq:x_dependent_characteristic_x}:
\begin{equation*}
\bx\left(\tau\right) = \bx\left(0\right) + \tau \nabla_\bp H\left(\bx\left(\tau\right),\nabla u\left(\bx\left(\tau\right),\tau\right)\right) + O\left(\tau^2\right)
\end{equation*}
for small $\tau\in\left[0,\Delta t\right]$.
For notational simplicity, let us denote
\begin{align*}
         \cS\left[u^k,u^{k-1}\right]\left(\bx,\tau\right)= & u^k\left(\bx,\tau\right) - \tau\nabla u^k\left(\bx,\tau\right)^{\text{T}}\nabla_\bp H\left(\bx,\nabla u^k\left(\bx,\tau\right)\right) \\
         & +\tau H\left(\bx,\nabla u^k\left(\bx,\tau\right)\right) - u^{k-1}\left(\bx-\tau\nabla_\bp H\left(\bx,\nabla u^k\right), \Delta t\right).
\end{align*}
\paragraph{Time Marching Algorithm}
Based on these, we propose the following time marching method that solves the HJ PDE \eqref{eq:x_dependent_hj} with the state dependent $H$ sequentially over short time intervals $\left[t_k,t_k+\Delta t\right]$:
\begin{enumerate}
    \item Set the initial condition $u^0 = g$.
    \item For $k=1,\cdots, N$,
    \begin{equation}\label{eq:loss_dt}
    u^k =  \underset{v}{\argmin}\int_0^{\Delta t}\int_\Omega 
\Bigl(\cS\left[v,u^{k-1}\right]\left(\bx,\tau\right)\Bigr)^2\diff\bx\diff t.
\end{equation}
\end{enumerate}
For each $k$, the predicted function $u^k$ approximates the solution $u$ of \eqref{eq:x_dependent_hj} on $t_k\leq t\leq t_{k+1}$, i.e.,
\begin{equation*}
    u^k\left(\bx,\tau\right)\approx u\left(\bx,k\Delta t+\tau\right),\ \forall \tau\in\left[0,\Delta t\right], \bx\in\Omega.
\end{equation*}

It is important to note that rather than using separate neural networks for each $u^k$, the model is trained using a single network, ensuring memory efficiency. After training the network for the solution $u^{k-1}$ on the time interval $\left[t_{k-1},t_{k-1}+\Delta t\right]$, the network parameters are saved. These saved parameters are then used as the initial function to train the same network for the subsequent solution $u^k$ by \eqref{eq:loss_dt}. As a result, when training $u^k$, the network is initialized with $u^{k-1}$, which accelerates the training process. Consequently, although the learning process is divided for time marching, the rapid convergence for each $u^k$ ensures that the overall training time does not increase significantly.

\section{Numerical Results}\label{sec:experiments}
In this section, we evaluate the performance of the proposed deep learning-based method for learning the implicit solution formula through a series of diverse examples and high-dimensional problems. Experiments are conducted on up to 40 dimensions, and both qualitative and quantitative results are presented. Although theoretical verification has not yet been established, extensive experiments on nonconvex Hamiltonians are also included, demonstrating the effectiveness of the proposed method in learning viscosity solutions. 

To assess the scalability of the proposed method, we maintain the same experimental configurations across all cases, regardless of dimensionality or domain size.
All experiments are conducted using an MLP \eqref{eq:MLP} of a depth $L=5$ and a width $64$ with the softplus activation function $\sigma\left(x\right)=\frac{1}{\beta}\log\left(1+e^{100 x}\right)$. The network is trained for for $N=200,000$ epochs using the gradient descent with an initial learning rate of $10^{-3}$ decayed by a factor of $0.99$ whenever the loss decreased. 
In each epoch, $M=5,000$ collocation points were uniformly randomly sampled from the domain. When boundary conditions are given, the regularization parameter $\lambda$ is set to 0.1 and the number of boundary collocation points $M_b$ is set to 200.
All experiments are implemented on a single NVIDIA GV100 (TITAN V) GPU.

\subsection{Convex Hamiltonians}
We begin by measuring the error with respect to the true solution for the theoretically validated convex (or concave) Hamiltonian. Experiments are conducted in up to 40 dimensions. In addition to accuracy, we also measure computational time and memory consumption to assess the efficiency of the approach for high-dimensional problems.
\begin{example}[Burgers' equation]\label{ex:burgers}
$\ $
    Consider the Burgers' equation with the quadratic Hamiltonian $H\left(\bp\right) = \frac{1}{2}\left\Vert\bp\right\Vert_2^2$ and initial function $g\left(\bx\right) = \left\Vert\bx\right\Vert_1$.
Experiments were conducted on the 1, 10, and 40 dimensions. The Mean Squared Error (MSE) with respect to the exact solution is summarized in the first row of \cref{tab:mse}.

To assess how the computational cost increases with dimension, we measure both computational time and memory consumption. The time taken for each training epoch was averaged over the total 10,000 training epochs, while the memory consumption is recorded as the maximum memory usage during a single epoch. The average values for these measurements across the three examples, including the two provided below, are reported in \cref{tab:mse}. It can be observed that neither computational time nor memory consumption increases sharply with dimension. It is important to emphasize that the computational time reported in \cref{tab:mse} refers to the time required to obtain the solution function over the entire spatio-temporal domain, rather than the time taken to compute the solution at discrete points or on a grid.
Moreover, as discussed in sub\cref{sec:inr}, the memory requirements of implicit neural representations are primarily determined by the size of the network. Increasing the dimension does not significantly alter the overall network size, except for the increase in the input dimension of the input layer. Consequently, the results in \cref{tab:mse} demonstrate that memory usage is nearly independent of dimensionality. These findings demonstrate that deep learning methods are highly scalable with respect to dimensionality, making them well-suited for addressing high-dimensional problems.
\end{example}
\begin{example}[Concave Hamiltonian] \label{ex:concave}
$\ $
    Consider the quadratic Hamiltonian $H\left(\bp\right) = -\frac{1}{2}\left\Vert\bp\right\Vert_2^2$ and initial function $g\left(\bx\right) = \left\Vert\bx\right\Vert_1$.
Experiments were conducted on the 1, 10, and 40 dimensions until time $T=1$.
MSE with respect to the exact solution is reported in the second row of \cref{tab:mse}.
\begin{table}[t]
    \centering
    \caption{The mean squared errors with the exact solution, the average computational time per epoch, and the memory usage for Examples \ref{ex:burgers}, \ref{ex:concave}, and \ref{ex:collision} across dimensions $d=1,10,40$ are presented. The computational time is measured as the average across three examples over a total of 10,000 epochs. Maximum memory consumption per iteration is measured. It is observed that the computational time and memory usage do
    not increase significantly as the dimension increases.} \label{tab:mse}
    \begin{tabular}{cccc}
    \toprule  
    Problem & $d=1$ & $d=10$ & $d=40$\\
    \midrule
    \cref{ex:burgers} & 1.14E-7 & 2.56E-5 & 1.30E-3\\
    \cref{ex:concave} & 8.59E-6 & 1.63E-4& 1.23E-3 \\
    \cref{ex:collision} & 7.08E-6 &  5.57E-5 & 1.13E-3\\
    Time (s) per Epoch & 0.01518 & 0.01630 & 0.01864\\
    Memory (MB) & 42.648 & 42.648 & 43.623\\
    \bottomrule
  \end{tabular}
\end{table} 
\end{example}

\begin{example}[Level set Propagation]\label{ex:collision}
    We consider the level set equation \cite{osher2004level} that governs the collision of two spheres, initially separated and moving along their respective normal directions, which ultimately results in a collision. 
    The level set propagation is written by
    \begin{equation*}
        u_t+\left\Vert\nabla u\right\Vert_2=0,
    \end{equation*}
    where the initial function $g$ is given as the signed distance function for two circles with centers at $(-0.3, 0, \cdots, 0)$ and $(0.3, 0, \cdots, 0)$, and radius of $0.2$.
    Experiments were conducted on the 1, 10, and 40 dimensions with $T=1$. 
    MSE with respect to the exact solution is summarized in the bottom row of \cref{tab:mse}. 
\cref{fig:collision} depicts the obtained solution for the two-dimensional case.
    \begin{figure}
    \centering
    \includegraphics[width=0.98\textwidth]{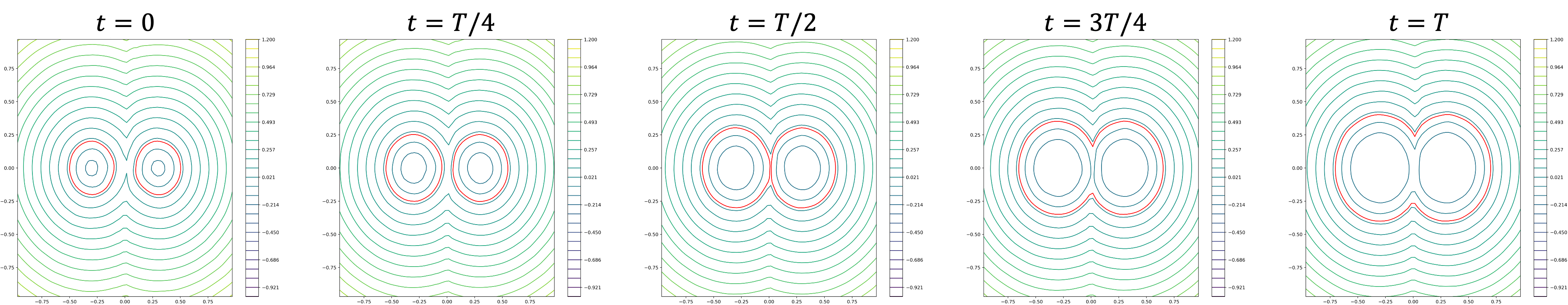}
    \caption{Iso-contours of the numerical solution to two-dimensional collision of circles in \cref{ex:collision}. The predicted zero leve lsets are represented by red curves.}
    \label{fig:collision}
\end{figure}
\end{example}

\subsection{Nonconvex Hamiltonians}\label{sec:exp_nonconvex}
In this subsection, we present experimental results for various Hamiltonians that are neither convex nor concave. While the theoretical proof for the proposed implicit solution formula has not yet been established in the nonconvex case, the experiments show that the proposed method effectively yields viscosity solutions.
\begin{example} \label{ex:cos}
We solve the nonlinear equation with a nonconvex Hamiltonian
\begin{equation*}
    u_t-\cos\left(\sum_{i=1}^d u_{x_i} + 1\right)=0,
\end{equation*}
with the initial function $g\left(\bx\right) = -\cos\left(\frac{\pi}{d}\sum_{i=1}^d x_i\right)$, and periodic boundary conditions presented in \cite{osher1991high}. 
The results for $d=1,2$ are depicted in \cref{fig:cos}.
The results are plotted up to time $T=0.2$, at which point kinks have already emerged in the solution.
\begin{figure}
    \centering
    \subfigure[$d=1$]{\includegraphics[width=0.97\textwidth]{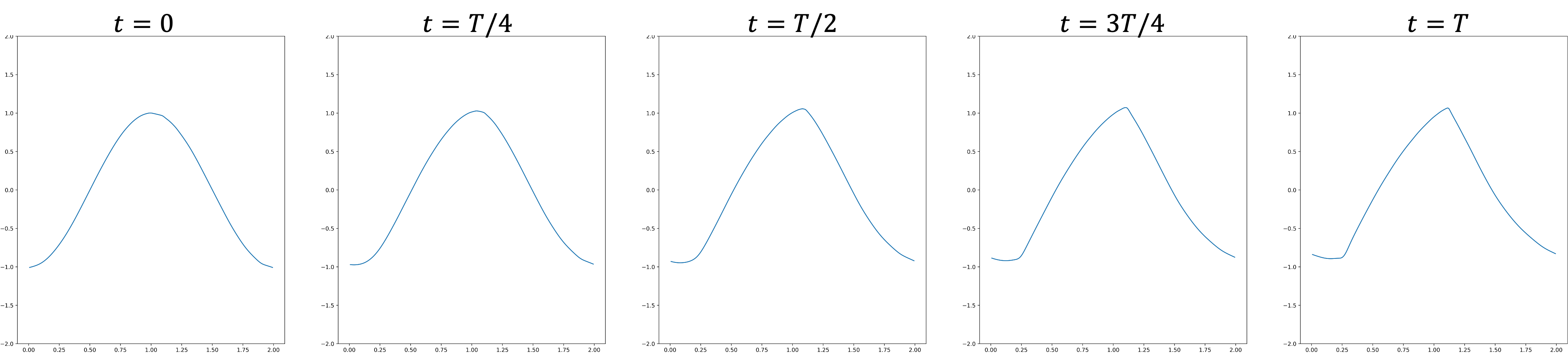}}$\quad$
	\subfigure[$d=2$]{\includegraphics[width=0.97\textwidth]{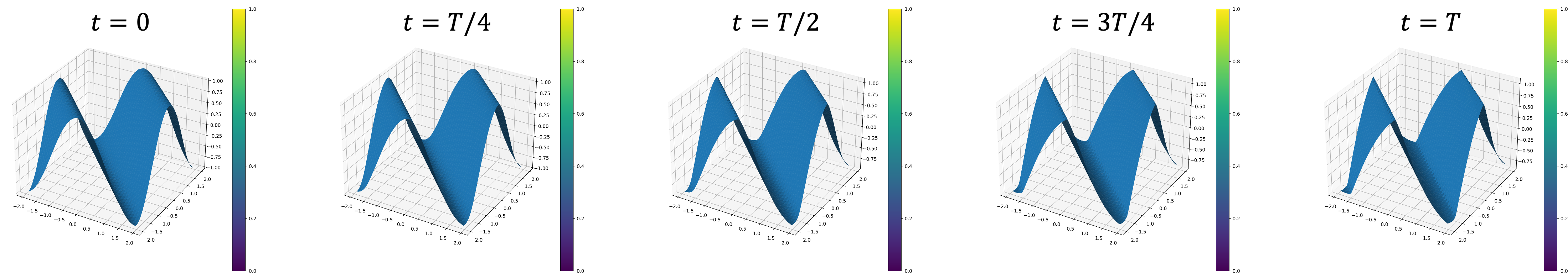}}
    \caption{The numerical results for \cref{ex:cos}.}
    \label{fig:cos}
\end{figure}
\end{example}

\begin{example}\label{ex:sin}
    The two-dimensional Riemann problem with a nonconvex flux \cite{osher1991high}
   \begin{equation*}
       \begin{cases}
           u_t+\sin\left(u_x+u_y\right)=0,\\
           u\left(x,y,0\right)=\pi\left(\left\vert y\right\vert - \left\vert x\right\vert\right).
       \end{cases}
   \end{equation*} 
   The predicted solution up to $T=1$ is depicted in \cref{fig:sin}.
   \begin{figure}
    \centering
    \includegraphics[width=0.97\textwidth]{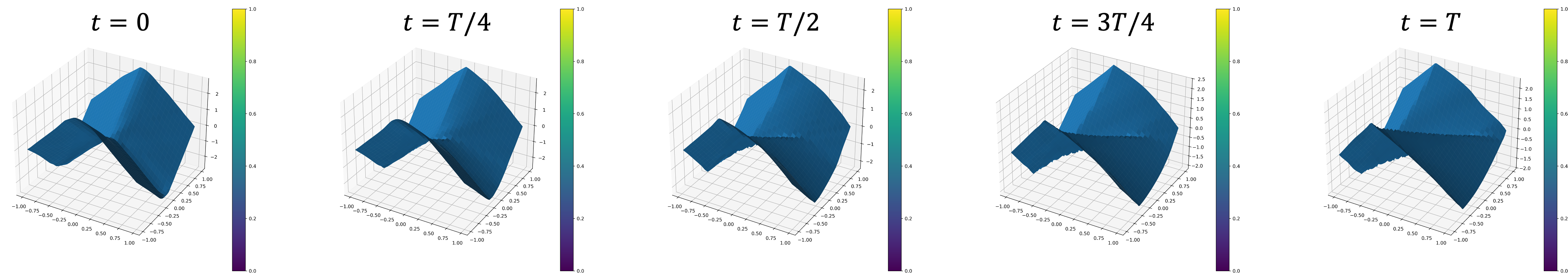}
    \caption{The numerical solution for \cref{ex:sin}}
    \label{fig:sin}
\end{figure}
\end{example}

\begin{example}\label{ex:multiply}
    The above nonconvex examples are actually one-dimensional along the diagonal. To evaluate the performance of the proposed method on fully two-dimensional problems \cite{bryson2003high}, we solve
    \begin{equation*}
        \begin{cases}
            u_t + u_xu_y = 0,\\
            u\left(x,y,0\right) = \sin(x)+\cos(y),
            \end{cases}
    \end{equation*}
    with periodic boundary condition and $T=1.5$. 
    The solution is smooth for $t<1$ and exhibits kinks for $t\geq 1$. The results presented in \cref{fig:multiply} show that the proposed method continues to accurately learn the solution even after the shock occurs.
    \begin{figure}
    \centering
    \includegraphics[width=0.98\textwidth]{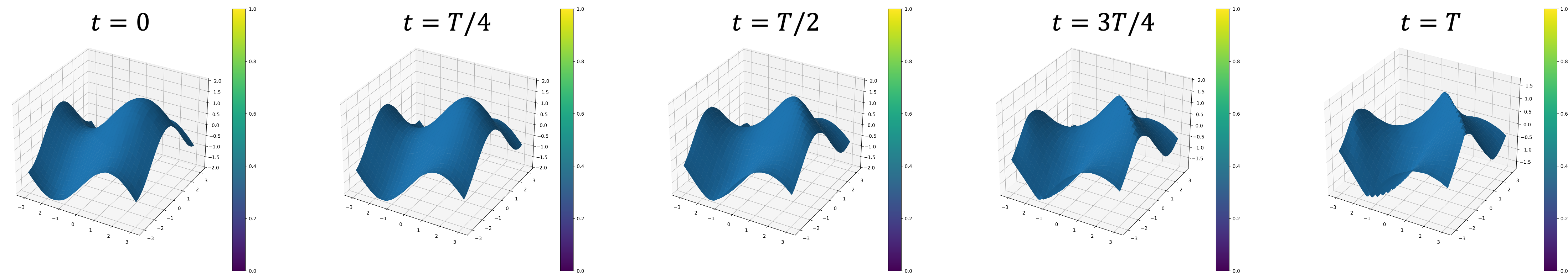}
    \caption{The numerical solution for \cref{ex:multiply}}
    \label{fig:multiply}
\end{figure}
\end{example}
\begin{example}[Eikonal equation]\label{ex:eikonal}
Consider a two-dimensional nonconvex problem \cite{osher1991high} that arises in geometric optics:
\begin{equation*}
    \begin{cases}
        u_t+\sqrt{u_x+u_y+1}=0,\\
        u\left(x,y,0\right) = \frac{1}{4}\left(\cos\left(2\pi x\right)-1\right)\left(\cos\left(2\pi y\right)-1\right) - 1.
    \end{cases}
\end{equation*}
The results up to time $T=0.45$ are shown in \cref{fig:eikonal}, where we can observe the sharp corners in the solution.
\begin{figure}
    \centering
    \includegraphics[width=0.98\textwidth]{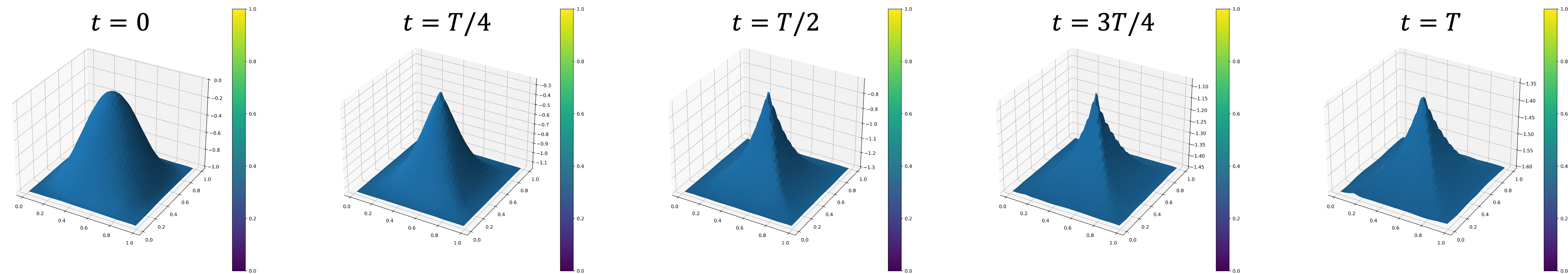}
    \caption{The numerical solution for \cref{ex:eikonal}}
    \label{fig:eikonal}
\end{figure}
\end{example}

\begin{example}\label{ex:combustion}
    Consider the combustion problem \cite{lafon1996high}:
    \begin{equation*}
    \begin{cases}
        u_t-\sqrt{u_x+u_y+1}=0,\\
        u\left(x,y,0\right) =\cos\left(2\pi x\right) - \cos\left(2\pi y\right).
    \end{cases}
\end{equation*}
Results up to time $0.27$ are given in \cref{fig:combustion}.
\begin{figure}
    \centering
    \includegraphics[width=0.98\textwidth]{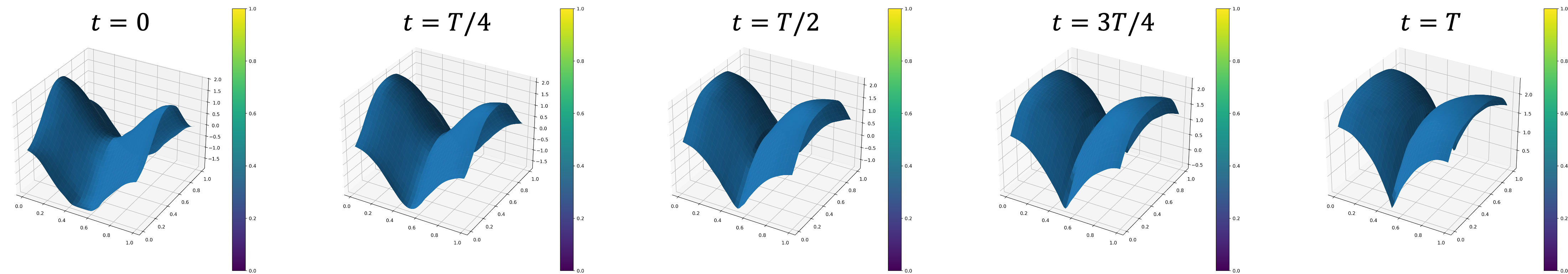}
    \caption{The numerical solution for \cref{ex:combustion}}
    \label{fig:combustion}
\end{figure}
\end{example}

\begin{example}\label{ex:cubic}
Consider the one-dimensional nonconvex problem
\begin{equation*}
    \begin{cases}
        u_t+ u_x^3-u_x=0,\\
        u\left(x,0\right) =-\frac{1}{10}\cos\left(5x\right).
    \end{cases}
\end{equation*}
The results up to time $T=0.7$ are shown in \cref{fig:cubic}, where we can observe can observe that the sinusoidal wave becomes progressively sharper.
\begin{figure}
    \centering
    \includegraphics[width=0.98\textwidth]{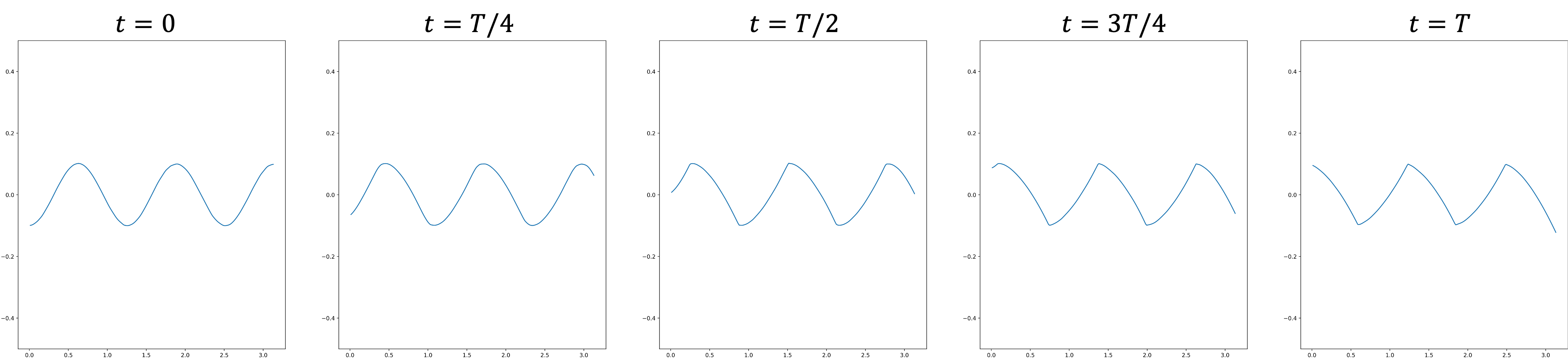}
    \caption{The numerical solution for \cref{ex:cubic}.}
    \label{fig:cubic}
\end{figure}
\end{example}

\subsection{State-dependent Hamiltonians}
This subsection provides experimental validation of the methodology proposed for the state-dependent Hamiltonian in sub\cref{sec:state_H}. It includes error analyses with respect to $\Delta t$ and addresses various state-dependent Hamiltonians, including a 10-dimensional optimal control problem.
\begin{example}\label{ex:state_sine}
We solve the following variable coefficient linear equation \cite{yan2011local}:
\begin{equation*}
    \begin{cases}
        u_t+ \sin \left(x\right) u_x=0,\\
        u\left(x,0\right) = \sin(x),
    \end{cases}
\end{equation*}
with periodic boundary condition.
The exact solution  is expressed by
\begin{equation*}
    u\left(x,t\right) = \sin\left(2\arctan\left(e^{-t}\tan\left(\frac{x}{2}\right)\right)\right).
\end{equation*}
We compute the solution up to $T=1$.
Since the characteristic curve for the state-dependent Hamiltonian is approximated linearly, the accuracy of the algorithm is influenced by the size of $\Delta t$. To verify this, we conducted experiments for various values of $\Delta t=0.5,0.25,0.1$, and the results are summarized in the top row of \cref{tab:state_mse}. The results show that the linear approximation of the proposed algorithm yields first-order accuracy with respect to $\Delta t$.

\begin{table}[t]
    \centering
    \caption{The mean squared errors (MSE) and relative mean square errors (RMSE) with the exact solution for Examples \ref{ex:state_sine} and \ref{ex:state_antisymm} with $\Delta t=0.5,0.25,0.1$.} \label{tab:state_mse}
    % \vspace{-5pt}
      % \setlength\tabcolsep{2.9pt}
    \begin{tabular}{c|cc|cc|cc}
    \toprule  
    Problem & \multicolumn{2}{c}{$\Delta t=0.1$} & \multicolumn{2}{c}{$\Delta t=0.25$} & \multicolumn{2}{c}{$\Delta t=0.5$}\\
   & MSE & RMSE & MSE & RMSE & MSE & RMSE\\
    \midrule
    \cref{ex:state_sine} & 3.27E-6 & 6.57E-6 & 8.82E-6 & 1.61E-5 & 1.26E-5 & 2.26E-5 \\
    \cref{ex:state_antisymm} & 6.29E-7 & 1.25E-3 & 3.88E-6 & 2.10E-3 & 6.12E-6 & 4.32E-3\\
    \bottomrule
  \end{tabular}
\end{table} 

\end{example}

\begin{example}\label{ex:state_antisymm}
We solve the following the two-dimensional linear equation which describes a solid body rotation around the origin \cite{cheng2007discontinuous}:
\begin{equation*}
    u_t -yu_x + xu_y=0,\ \left(x,y\right)\in\left(-1,1\right)^2
\end{equation*}
where the initial condition is given by
\begin{equation*}
   g\left(x,y\right)= \begin{cases}
        0 & 0.3 \leq r,\\
       0.3 - r & 0.1<r<0.3\\
       0.2 & r\leq 0.1,
    \end{cases}
\end{equation*}
where $r=\sqrt{(x-0.4)^2 + (y-0.4)^2}$.
We also impose the periodic boundary condition.
The exact solution is 
\begin{equation*}
    u\left(x,y,t\right) = g\left(x\cos t + y\sin t, -x\sin t + y\cos t\right).
\end{equation*}
We compute the solution up to $T=1$ and the numerical errors for $\Delta t=0.5,0.25,0.1$ are reported in the second row of \cref{tab:state_mse}.
As in the previous example, we observe that the error increases linearly as $\Delta t$ increases.
\end{example}

\begin{example}\label{ex:oc}
We solve an optimal control problem related to cost determination \cite{osher1991high}:
\begin{equation*}
    \begin{cases}
        u_t + u_x \sin y + \left(\sin y + \text{sign}\left(u_y\right)\right)u_y - \frac{1}{2}\sin^2 y - \left(1-\cos x\right)=0,\\
        u\left(x,y,0\right) = 0,
    \end{cases}
\end{equation*}
with periodic conditions. The result at $T=1$ is presented in \cref{fig:oc} and is qualitatively in agreement with \cite{osher1991high}.

\begin{figure}
    \centering
    \includegraphics[width=0.4\textwidth]{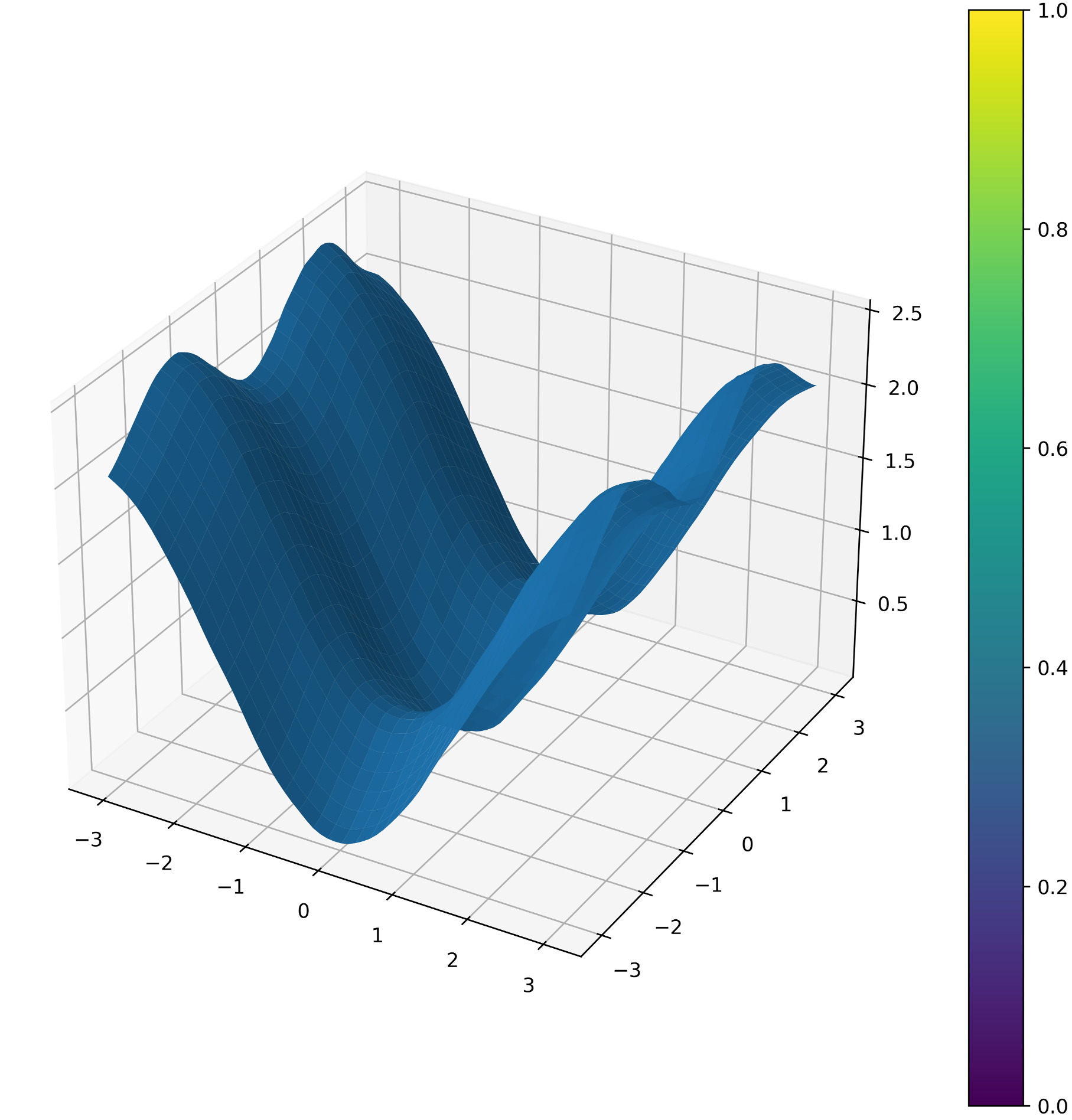}
    \caption{The numerical solution for \cref{ex:oc}.}
    \label{fig:oc}
\end{figure}
    
\end{example}

\begin{example}\label{ex:harmonic_oscillator}
    We solve the problem associated with the state-dependent Hamiltonian well-known as the harmonic oscillator:
    \begin{equation*}
        H^{\pm}\left(\bx,\bp\right)=\pm\frac{1}{2}\left(\left\Vert\bx\right\Vert_2^2 + \left\Vert\bp\right\Vert_2^2\right).
    \end{equation*}
    We consider the two-dimensional problem where the initial function is the level set function of an ellipsoid 
    \begin{equation}\label{eq:ellipse_initial_ft}
    g\left(x,y\right) = \frac{1}{2}\left(\frac{x^2}{2.5^2} + y^2-1\right).
    \end{equation}   
The results up to $T=0.4$ are depicted in \cref{fig:harmonic_oscillator}.
\begin{figure}
    \centering
    \subfigure[$H^+$]{\includegraphics[width=0.97\textwidth]{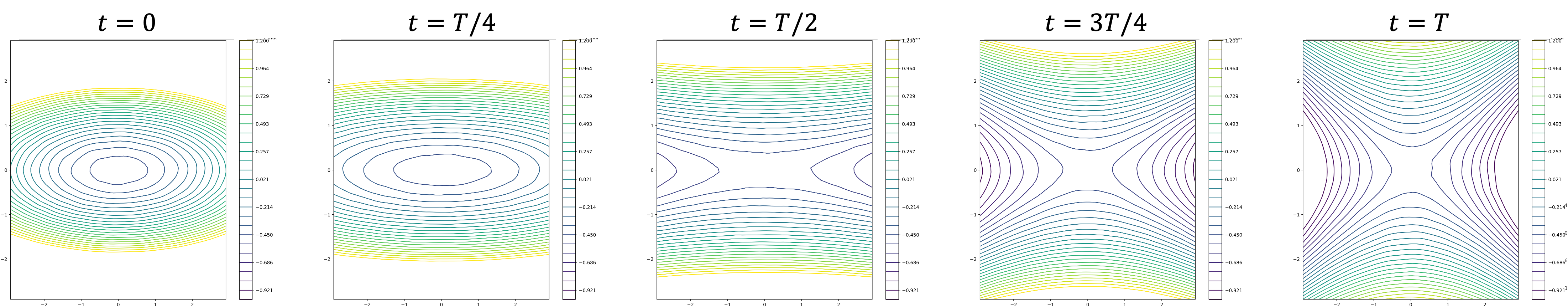}}$\quad$
	\subfigure[$H^-$]{\includegraphics[width=0.97\textwidth]{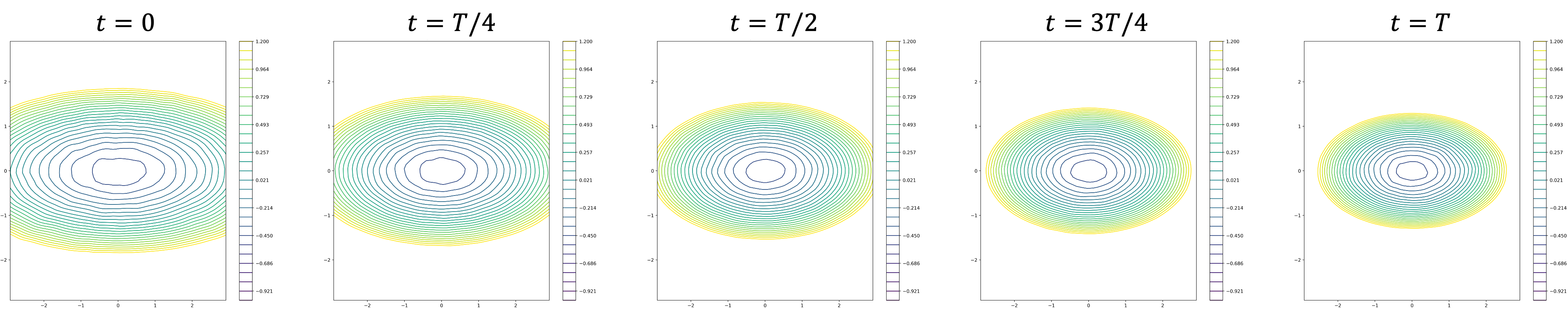}}
    \caption{The numerical results for \cref{ex:harmonic_oscillator}.}
    \label{fig:harmonic_oscillator}
\end{figure}
\end{example}

\begin{example}\label{ex:state_nonconvex1}
    We consider a state-dependent nonconvex Hamiltonian of the following form given in \cite{chow2019algorithm}:
    \begin{equation*}
        H\left(\bx,p\right) = -c\left(\bx\right)p_1 + 2\left\vert p_2\right\vert + \left\Vert \bp\right\Vert_2 - 1,
    \end{equation*}
    where $\bp=\left(p_1,p_2\right)$ and 
    \begin{equation}\label{eq:coeff_c}
        c\left(\bx\right) = 2\left(1+3\exp\left(-4\left\Vert \bx-\left(1,1\right)\right\Vert^2_2\right)\right).
    \end{equation}
    We employ the initial function as presented in \cref{ex:harmonic_oscillator}.     
The results up to $T=1$ are presented in \cref{fig:state_nonconvex1}, which are consistent with those reported in \cite{chow2019algorithm}.
    \begin{figure}
    \centering
    \includegraphics[width=0.98\textwidth]{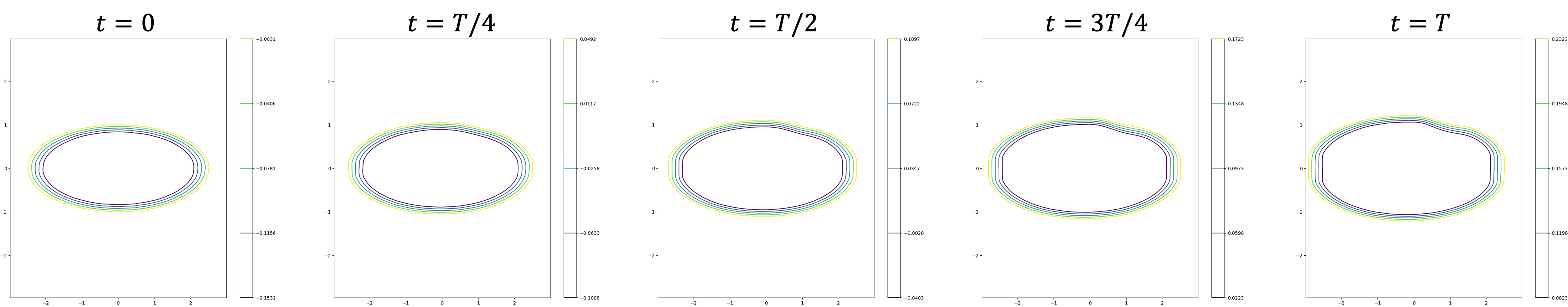}
    \caption{The numerical solution for \cref{ex:state_nonconvex1}.}
    \label{fig:state_nonconvex1}
\end{figure}
\end{example}

\begin{example}\label{ex:state_nonconvex2}
    We test the proposed method for a state-dependent nonconvex Hamiltonian of the following form given in \cite{chow2019algorithm}:
    \begin{equation*}
        H\left(\bx,p\right) = -c\left(\bx\right)\left\vert p_1\right\vert -c\left(-\bx\right)\left\vert p_2\right\vert,
    \end{equation*}
    where we write $\bp=\left(p_1,p_2\right)$ and 
    $c\left(x\right)$ is a coefficient function as given in \eqref{eq:coeff_c}.
     The initial function $g$ \eqref{eq:ellipse_initial_ft} presented in \cref{ex:harmonic_oscillator} is employed in this instance.
The results up to $T=0.3$ are presented in \cref{fig:state_nonconvex2} and the results are in agreement with those reported in \cite{chow2019algorithm}.
    \begin{figure}
    \centering
    \includegraphics[width=0.98\textwidth]{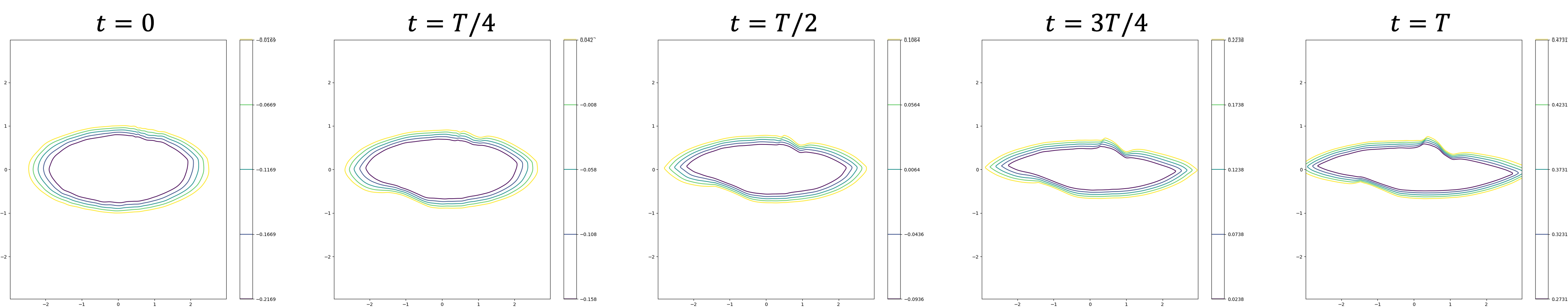}
    \caption{The numerical solution for \cref{ex:state_nonconvex2}.}
    \label{fig:state_nonconvex2}
\end{figure}
\end{example}

\begin{example}\label{ex:state_oc}
We solve the following optimal control problem:
\begin{equation*}
        u\left(\bx,t\right) = \inf\Bigl\{ g\left(\bx\left(0\right)\right); \dot{\bx}\left(t\right)=f\left(\bx\left(t\right)\right)\ba\left(t\right),\ \bx\left(t\right) = \bx,\ \left\Vert\ba\left(t\right)\right\Vert_2 \leq 1\Bigr\},
    \end{equation*}
        where $g$ is defined by
        \begin{equation*}
            g\left(\bx\right) = \frac{1}{2}\left(\bx^{\text{T}}A\bx - 1\right)
        \end{equation*} 
        with $A=\text{diag}(0.25,1)$ and $f$ is given by
        \begin{equation*}
        f\left(\bx\right) = 1+3\exp\left(-4\left\Vert \bx-\left(1,1\right)\right\Vert^2_2\right).
    \end{equation*}
   This corresponds to the HJ PDE given in \cite{chow2019algorithm}:
    \begin{equation*}
        \begin{cases}
            u_t+f\left(\bx\right)\left\Vert \nabla u\right\Vert_2=0\\
            u\left(\bx,0\right)=g\left(\bx\right).
        \end{cases}
    \end{equation*}
    When solving the maximization problem $\sup\ g\left(\bx\left(T\right)\right)$ with the same constraints, we obtain the following HJ PDE:
    \begin{equation*}
        \begin{cases}
            u_t-f\left(\bx\right)\left\Vert \nabla u\right\Vert_2=0\\
            u\left(\bx,0\right)=g\left(\bx\right).
        \end{cases}
    \end{equation*}
    The results for both the minimization (at $T=0.2$) and maximization (at $T=0.5$) problems are presented in \cref{fig:state_linear_oc} and are consistent with the results in \cite{chow2019algorithm}.
   \begin{figure}
    \centering
    \centering
    \subfigure[Minimization]{\includegraphics[width=0.3\textwidth]{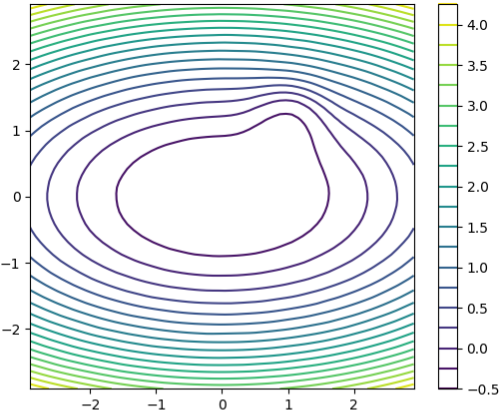}}$\quad\quad$
	\subfigure[Maximization]{\includegraphics[width=0.3\textwidth]{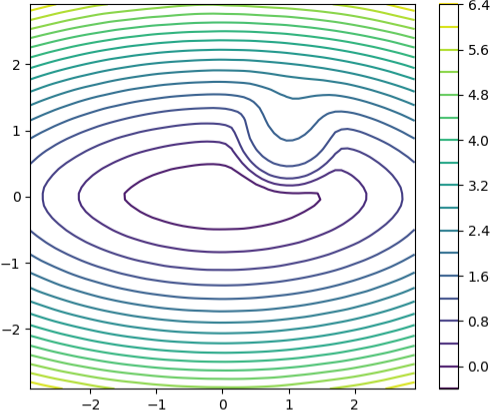}}
    \caption{The numerical solution for \cref{ex:state_oc} with $\Delta t=0.1$.}
    \label{fig:state_linear_oc}
\end{figure}
    \end{example}
    
\begin{example}\label{ex:state_oc_quad}
    Consider the following 10-dimensional quadratic optimal control problem presented in \cite{chen2024hopf}:
    \begin{equation*}
        u\left(\bx,t\right) = \inf\Bigl\{\int_0^t\left\Vert \dot{\bx}\left(s\right)\right\Vert ^2 - \psi\left(\bx\left(s\right)\right)\diff s + g\left(\bx\left(0\right)\right); \bx\left(t\right) = \bx\Bigr\},
    \end{equation*}
    where the potential function $\psi:\bR^d\rightarrow \left(-\infty,0\right]$ is $\psi\left(\bx\right) = \sum_{i=1}^d \psi_i\left(\bx_i\right)$, where each function $\psi_i:\bR\rightarrow \left(-\infty,0\right]$ is a positively 1-homogeneous concave function given by
    \begin{equation*}
        \psi_i\left(x\right)=\begin{cases}
            -a_ix & x\geq 0,\\
            b_ix & x<0,
        \end{cases}
    \end{equation*}
    with parameters 
    $\left(a_1,\cdots, a_d\right) = \left(4,6,5,\cdots,5\right)$ and $\left(b_1,\cdots,b_d\right) = \left(3,9,6,\cdots,6\right)$.
The corresponding HJ PDE reads:
\begin{equation*}
    \begin{cases}
        u_t+\frac{1}{2}\left\Vert \nabla u\right\Vert ^2 + \psi\left(\bx\right)= 0\\
        u\left(\bx,0\right) = g\left(\bx\right).
    \end{cases}
\end{equation*}
We conduct experiments for the two initial cost functions:
\begin{itemize}
    \item A quadratic initial function
    $g_1\left(\bx\right) = \frac{1}{2}\left\Vert \bx - \mathbf{1}\right\Vert_2^2$, where $\mathbf{1}$ denotes the $d$-dimensional vector whose elements are all one.\\
    \item A nonconvex initial function 
    \begin{equation*}
        g_2\left(\bx\right) = \underset{j\in\{1,2,3\}}{\min}\ g_j\left(\bx\right) = \underset{j\in\{1,2,3\}}{\min}\ \frac{1}{2}\left\Vert \bx-\by_j\right\Vert_2^2 - \alpha_j,
    \end{equation*}
    where $\by_1=\left(-2,0,\cdots,0\right)$, $\by_2 = \left(2,-2,-1,0,\cdots,0\right)$, $\by_3=\left(0,2,0,\cdots,0\right)$, $\alpha_1=-0.5$, $\alpha_2=0$, and $\alpha_3=-1$.
\end{itemize}
\end{example}
\cref{fig:state_oc_quad} presents two-dimensional slices of the solutions in the $xy$ plane for both cases up to time $T=0.5$. 
\begin{figure}
    \centering
    \subfigure[$g_1$]{\includegraphics[width=0.97\textwidth]{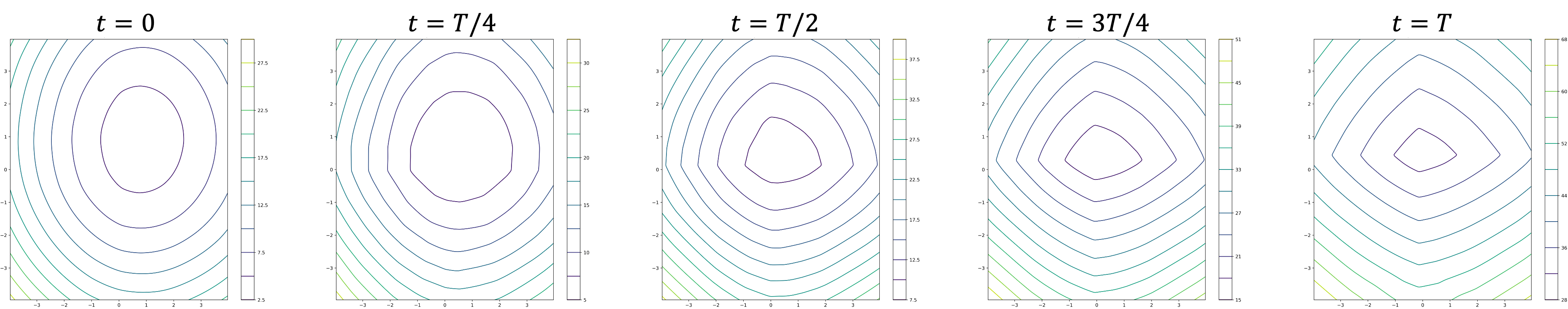}}$\quad$
	\subfigure[$g_2$]{\includegraphics[width=0.97\textwidth]{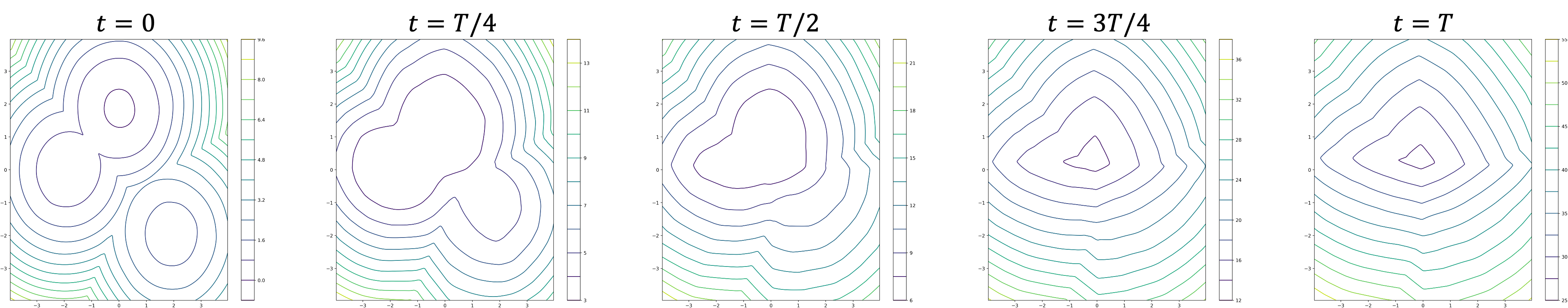}}
    \caption{The numerical results for \cref{ex:state_oc_quad}.}
    \label{fig:state_oc_quad}
\end{figure}
The results demonstrate that the evolution of the solution is non-trivial, as evidenced by the nonlinear progression of the level sets over time, exhibiting multiple kinks. These findings are consistent with the experimental results presented in \cite{chen2024hopf}.

\section{Conclusion}
We have introduced a novel implicit solution method for HJ PDEs derived from the characteristics of the PDE. This formula aligns with the Hopf-Lax formula for convex Hamiltonians but simplifies it by removing the need for Legendre transforms, thereby enhancing computational efficiency and broadening its practical applicability. The proposed formula not only bridges the method of characteristics, the Hopf-Lax formula, and Bellman's principle from control theory but also offers a simple and effective numerical approach for solving HJ PDEs. By integrating deep learning, the formula provides a scalable method that effectively mitigates the curse of dimensionality. Experimental results demonstrate its robustness and effectiveness across various high-dimensional and nonconvex problems without tuning the configuration of the deep learning model. These findings validate the method as a versatile and computationally efficient tool for solving high-dimensional, nonconvex dynamic systems and optimal control problems governed by HJ PDEs.

An important direction for future work includes a rigorous analysis of the proposed implicit solution formula. While experimental results demonstrate the method's effectiveness on various nonconvex problems, a comprehensive analysis is needed to confirm whether the proposed formula describes the viscosity solution of HJ PDEs in nonconvex problems. Since the formula involves the first derivatives and is a composite of multiple terms, the proposed minimization problem \eqref{eq:implicit_loss} is nonconvex, making the convergence of gradient descent non-trivial. Consequently, a convergence analysis would be an important future endeavor.

Regarding the deep learning approach, we approximate the expectation loss \eqref{eq:implicit_loss} using Monte Carlo integration \eqref{eq:discrete_loss}, which introduces a discrepancy between the empirical and expectation losses. A valuable research direction could involve investigating whether the stochastic gradient descent process, with its random collocation points at each epoch, converges to the global minimum of the expectation loss in the context of stochastic approximation.
Although we focused on scalability by maintaining a fixed model configuration across experiments, future research should explore the optimal selection of collocation points and network size for different problem dimensions. Furthermore, the investigation of using automatic differentiation to compute exact derivatives of the network, rather than finite differences such as ENO/WENO, presents an intriguing avenue for future research, particularly in its ability to capture shocks.
For state-dependent Hamiltonians, the development of higher-order methods beyond the proposed first-order linear approximation of the characteristic curve would be a promising direction. Finally, the simplicity and efficiency of the proposed method open up avenues for its application to a wide range of problems, including level set evolutions, optimal transport, mean field games, and inverse problems, which would constitute valuable extensions of this work.

\section*{Acknowledgments}
We would like to express our sincere gratitude to Tingwei Meng for the valuable discussions on the implicit formula during the early stages of our manuscript.

\bibliographystyle{siamplain}
\bibliography{mybib}

\end{document}